%% file: aaai2026.tex
\documentclass[letterpaper]{article} % DO NOT CHANGE THIS
\usepackage{aaai2026}  % DO NOT CHANGE THIS
\usepackage{times}  % DO NOT CHANGE THIS
\usepackage{helvet}  % DO NOT CHANGE THIS
\usepackage{courier}  % DO NOT CHANGE THIS
\usepackage[hyphens]{url}  % DO NOT CHANGE THIS
\usepackage{graphicx} % DO NOT CHANGE THIS
\usepackage{natbib}  % DO NOT CHANGE THIS AND DO NOT ADD ANY OPTIONS TO IT
\usepackage{caption} % DO NOT CHANGE THIS AND DO NOT ADD ANY OPTIONS TO IT
\usepackage{algorithm}
\usepackage{algorithmic}
\usepackage{tikz}
\usetikzlibrary{decorations.pathreplacing}

\usepackage{newfloat}
\usepackage{listings}
\DeclareCaptionStyle{ruled}{labelfont=normalfont,labelsep=colon,strut=off} % DO NOT CHANGE THIS
\floatstyle{ruled}
\newfloat{listing}{tb}{lst}{}
\floatname{listing}{Listing}
\usepackage{xcolor}
\usepackage{xspace}
\newcommand{\sys}{PSEO\xspace}
\usepackage{booktabs}
\usepackage{amsmath}
\usepackage{amsthm}
\usepackage{amssymb}
\usepackage{multirow}
\usepackage{makecell}
\usepackage{subfigure}
\usepackage{subcaption}
\usepackage{appendix}
\usepackage{rotating}

\newtheorem{theorem}{Theorem}

\newtheorem{lemma}{Lemma}

\title{\sys: Optimizing Post-Hoc Stacking Ensemble Through \\Hyperparameter Tuning}
\author{
    Beicheng Xu\textsuperscript{\rm 1}, Wei Liu\textsuperscript{\rm 1}, Keyao Ding\textsuperscript{\rm 1}, Yupeng Lu\textsuperscript{\rm 1}, Bin Cui\textsuperscript{\rm 1}\thanks{Bin Cui is the corresponding author}
}
\affiliations{
    \textsuperscript{\rm 1}School of CS \& Key Laboratory of High Confidence Software \\Technologies (MOE), Peking University, Beijing, China\\
    \{beichengxu, eularioal, maodeshi\}@stu.pku.edu.cn, 
    \{xinkelyp, cui.bin\}@pku.edu.cn
}

\usepackage{bibentry}
\begin{document}

\maketitle

\begin{abstract}
The Combined Algorithm Selection and Hyperparameter Optimization (CASH) problem is fundamental in Automated Machine Learning (AutoML). Inspired by the success of ensemble learning, recent AutoML systems construct post-hoc ensembles for final predictions rather than relying on the best single model.
% However, while most CASH methods conducted an extensive search for the optimal single model, they typically use a fixed strategy during the ensemble phase, which does not adapt to the specific task characteristics.
However, while most CASH methods conduct extensive searches for the optimal single model, they typically employ fixed strategies during the ensemble phase that fail to adapt to specific task characteristics.
To tackle this issue, we propose \sys, a framework for post-hoc stacking ensemble optimization. 
First, we conduct base model selection through binary quadratic programming, with a trade-off between diversity and performance.
Furthermore, we introduce two mechanisms to fully realize the potential of multi-layer stacking.
Finally, \sys builds a hyperparameter space and searches for the optimal post-hoc ensemble strategy within it. 
Empirical results on 80 public datasets show that \sys achieves the best average test rank (2.96) among 16 methods, including post-hoc designs in recent AutoML systems and state-of-the-art ensemble learning methods.
\end{abstract}
% Uncomment the following to link to your code, datasets, an extended version or similar.
% You must keep this block between (not within) the abstract and the main body of the paper.
\begin{links}
    \link{Code}{https://github.com/PKU-DAIR/mindware}
    % \link{Datasets}{https://aaai.org/example/datasets}
    \link{Extended version}{https://arxiv.org/pdf/2508.05144}
\end{links}
\input{tex/introduction}

\input{tex/background}

\input{tex/method}

\input{tex/experiment}

\input{tex/relatedworks}

\section{Conclusion}

In this paper, we propose \sys, an efficient optimization framework to tune the post-hoc stacking ensemble.
In \sys, we proposed three components: a base model subset selection algorithm with a trade-off between individual model performance and inter-model diversity, a deep stacking ensemble with Dropout and Retain mechanisms to fully explore the potential of multi-layer stacking, and finally a Bayesian optimizer to search for the optimal ensemble strategy.
We evaluated \sys on 80 public datasets and demonstrated its superiority over competitive baselines.

\section{Acknowledgement}
This work is supported by National Natural Science Foundation of China (U23B2048, U22B2037).

\bibliography{aaai2026}

\input{tex/reproductionchecklist}

\clearpage
\input{tex/appendix}

\end{document}

%% file: tex/introduction.tex
\section{Introduction}

In recent years, machine learning has advanced in diverse application domains, such as computer vision~\cite{he2016deep,JCST-2309-13814}, and recommendation systems~\cite{JCST-2101-11277,chen2025privacy}. Despite these advancements, developing tailored solutions with strong performance remains a knowledge-intensive task, requiring careful selection of suitable ML algorithms and hyperparameter tuning.
To lower barriers and streamline the deployment of machine learning applications, the AutoML community has introduced the Combined Algorithm Selection and Hyperparameter Optimization (CASH) problem~\cite{thornton2013auto} and proposed several methods~\cite {hutter2019automated,he2021automl} to automate this optimization process.
% Specifically, the CASH problem seeks to identify the best combination of machine learning algorithms and their corresponding hyperparameters for a given task, with the aim of maximizing model performance.
Furthermore, it’s widely acknowledged ensembling promising models often outperforms single ones ~\cite{feurer2015efficient,luo2024survey,zhu2025relational}.  
And ensemble strategies can be frequently found in the top solutions of prediction competitions~\cite{koren2009bellkor,hoch2015ensemble}.
% Although methods like Auto-WEKA~\cite{thornton2013auto} and EO~\cite{levesque2016bayesian} attempt to directly search for an ensemble solution in a single step, they are constrained by the vast search space, which includes both the ensemble strategy and the type and hyperparameters of each base model.
As a result, recent AutoML systems~(e.g., Auto-sklearn~\cite{feurer2022auto}, AutoGluon~\cite{erickson2020autogluon}, and VolcanoML~\cite{li2023volcanoml}) build post-hoc ensembles using base models from the optimization process of the optimal individual model and show better empirical results than the best single models.

% solve it in two steps: first, they search for the optimal individual model, and then adopt a specific ensemble strategy based on the observation history generated in the first step, which we refer to as the post-hoc ensemble designs. 
% The post-hoc ensemble design avoids exploring overly large search spaces and shows better empirical results than searching for an ensemble in one step or using the best single learner.

\begin{figure}[t!]
	\centering
	% \scalebox{0.95}[.95] {
	\includegraphics[width=\linewidth]{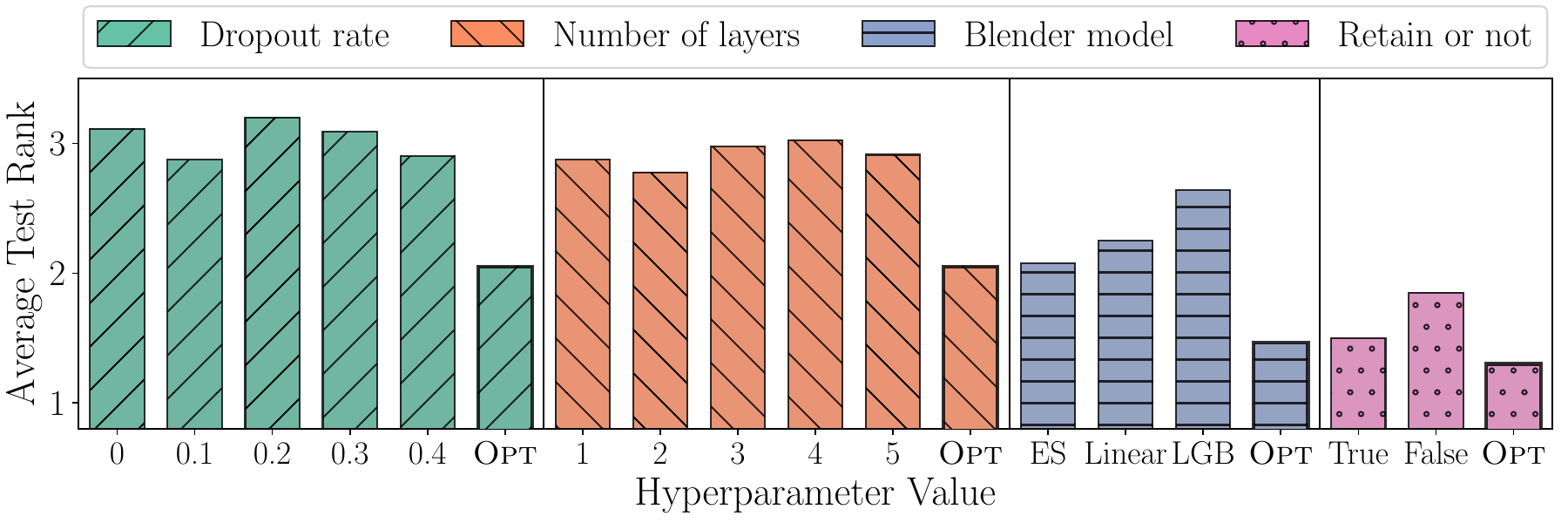}
         % }
        % \vspace{-2em}
	\caption{Average test rank of staking across each hyperparameter's values with other parameters fixed.}
        \label{fig:struc}
        % \vspace{-1em}
\end{figure}
% Dropout and Retain are two mechanisms proposed by \sys.

% \begin{table*}[t]
%     \vspace{0mm}
%     \centering
%     % \caption{\small Post-hoc Ensemble Strategy of Current AutoML System.}
    
%     \begin{tabular}{lccc}
%     \toprule
%         \textbf{AutoML System} & \textbf{Base Model Selection} & \textbf{Number of layers} &\textbf{Blender Model}\\
%     \midrule
%         AutoGluon           & All       & Multi-Layer & Ensemble Selection(ES)  \\ 
%         LightAutoML         & Best of Each Class                  & One/Two-Layer   & Weighted Avaraging  \\ 
%         H2O                 & All/Best of Each Class                  & One-Layer   & Linear/LightGBM(LGB)/... \\
%         % Auto-Sklearn        & User Defined      & Greedy                              & Ensemble Selection \\
%         % Auto-Pytorch        & User Defined      & Greedy                              & Ensemble Selection \\
%         % MLJAR               & All               & \textbackslash                      & Ensemble Selection \\
%     \bottomrule
%     \end{tabular}
    
%     \caption{Stacking strategy of current AutoML systems.}
%     \label{tab:automl_sys_ens}
%     % \vspace{3mm}
% \end{table*}

Despite the effectiveness of those post-hoc ensemble designs, they typically employ a fixed ensemble strategy or leave it to the user to specify, rather than performing comprehensive optimization like with a single optimal model. 
However, as shown in Figure~\ref{fig:struc} and detailed in Appendix~\ref{app:exp_effectiveness_of_struc}, when we apply stacking ensembles, fixing other hyperparameters and performing a grid search for each hyperparameter separately, the tuned value according to validation performance~(\textsc{Opt}) achieve higher average test ranks than other fixed-value strategies, demonstrating the potential of optimizing ensembles.
% This means that achieving truly better ensemble solutions still necessitates substantial experimentation time and expert knowledge.
% This means that truly effective ensemble solutions still require significant experimentation and expert insight.
% with post-hoc ensembles 
Additionally, existing AutoML systems seldom explore base model selection for ensembles; instead, all available models or the best subset are simply used.
% On one hand, both theoretical analyses and empirical studies have supported the concept that "many-could-be-better-than-all" when selecting models for an ensemble~\cite{martinez2008analysis}. On the other hand, selecting only the best subset lacks the assurance of base model diversity, which is detrimental to the final ensemble performance~\cite{lecun2015deep}.
Therefore, how to search for the optimal post-hoc ensemble solution, including base model selection and strategies for ensemble base models, remains an unresolved problem.

The optimization of post-hoc ensemble is non-trivial.
When we attempt to characterize an ensemble strategy using hyperparameters and search for an optimal ensemble, two challenges arise:
\textbf{C1}. Since the input for post-hoc ensembles is a discrete pool of pre-trained base models, selecting models for integration is a combinatorial optimization problem, where the hyperparameter space can be exceedingly large.
As a result, it is challenging to reduce the number of hyperparameters required to characterize the problem, enabling more efficient exploration.
% \textbf{C2}. It is challenging to fully realize the potential of deep stacking due to the risk of overfitting by excessively relying on certain base models. Moreover, as layers increase, the stacker model's stability can decrease, limiting the depth of stacking.
\textbf{C2}. It is challenging to fully realize the potential of the post-hoc ensemble and provide a flexible framework adaptable to various tasks, along with optional strategies for addressing specific issues.

In this paper, we propose \sys, a new framework that optimizes post-hoc stacking ensemble through hyperparameter tuning.
The contributions are summarized as follows: 1) To the best of our knowledge, \sys is the first work to optimize post-hoc stacking ensemble through hyperparameter tuning.
2) To deal with \textbf{C1}, \sys transforms the base model selection task into a binary quadratic programming problem and uses only two hyperparameters to control the optimization objective, achieving a trade-off between the diversity of selected base models and individual performance.
3) For \textbf{C2}, \sys introduces the Dropout and Retain mechanisms to mitigate the problems of overfitting and feature degradation, respectively.
% 4) Instead of exhibiting fixed strategies, \sys builds a hyperparameter space for post-hoc stacking ensemble, and uses a Bayesian optimizer to iteratively search for the optimal ensemble strategy within it.
% 5) We evaluate \sys on 80 public datasets. Empirical results show that \sys achieves the best average rank (2.96) on test errors among 16 methods, including post-hoc designs in recent AutoML systems and state-of-the-art methods for ensemble learning on CASH problems.
4) Instead of relying on fixed strategies, \sys defines a hyperparameter space for post-hoc stacking ensembles and employs Bayesian optimization to search for the optimal strategy iteratively.
Empirical results on 80 public datasets show that \sys achieves the best average test rank (2.96) among 16 methods.
%including post-hoc designs and state-of-the-art ensemble learning methods.

% We summarize our contributions as follows:
% \begin{itemize}
%     \item To the best of our knowledge, \sys is the first work to optimize post-hoc stacking ensemble through hyperparameter tuning.
%     \item To enhance the flexibility and potential of the post-hoc ensemble,  we 1) transform base model selection into a binary quadratic optimization problem; 2) identify issues of overfitting and feature degradation in deep stacking and propose two mechanisms to mitigate them.
%     \item To optimize post-hoc ensemble for diverse tasks, we first use hyperparameters to control the solution of the quadratic problem and the structure of the stacking. Then we adopt Bayesian optimization to iteratively search for the optimal ensemble strategy within the search space.
%     \item We evaluate \sys on 80 public datasets. Empirical results show that \sys achieves the best average rank (2.96) on test errors among 16 methods, including post-hoc designs in recent AutoML systems and state-of-the-art methods for ensemble learning on CASH problems.
% \end{itemize}

%% file: tex/background.tex
\section{Background}
\label{sec:background}

\subsection{Post-Hoc Ensemble}

In post-hoc ensemble, we are given a fitted model pool $\mathbb{M}=\{m_1, m_2, ..., m_n\}$, which consists of all base models that are trained on $\mathcal{D}_{train}$ and validated during the search for the optimal individual model.
The goal is to aggregate an optimal subset of models from $\mathbb{M}$ with an ensembler $f(\cdot; \theta)$ to minimize a loss function $\mathcal{L}$ in validation set $\mathcal{D}_{val}$, which is:
% \begin{equation}
%     \min_{m_1, ..., m_{n'} \in \mathbb{M};\ \theta} \mathcal{L}(f(m_1, ..., m_{n'}; \theta), \mathcal{D}_{val})
% \end{equation}
\begin{equation}
    \min_{m_{i_1}, ..., m_{i_{n'}} \in \mathbb{M};\ \theta} \mathcal{L}(f(m_{i_1}, m_{i_2}, ..., m_{i_{n'}}; \theta), \mathcal{D}_{val})
\end{equation}
Intuitively, a post-hoc ensemble needs to address two key problems: (1) how to select an appropriate subset of base models from the candidate pool $\mathbb{M}$, and (2) how to determine the best ensembler $f(\cdot; \theta)$ to produce the final output.

\subsection{Stacking}
Among different ensemble strategies (e.g., bagging~\cite{dietterich2000ensemble}, boosting~\cite{freund1996experiments}, ensemble selection~\cite{caruana2004ensemble}),
we adopt stacking~\cite{van2007super} as the ensembler because: 
1) It has shown excellent experimental results, outperforming other methods~\cite{puruckerassembled}. 
2) Stacking offers inherent flexibility.
3) It's a general framework, as commonly used methods can be considered special cases of one-layer stacking and serve as the blender model.

Recently, some AutoML systems like AutoGluon~\cite{erickson2020autogluon}, LightAutoML~\cite{vakhrushev2021lightautoml}, and H2O~\cite{ledell2020h2o} adopt stacking for post-hoc ensemble, which trains a blender model to aggregate the predictions of a set of base models~\cite{breiman1996stacked}. 
%  The key insight of stacking is that the stacker model can improve upon the shortcomings of individual base predictions and leverage interactions between base models to enhance predictive performance~\cite{van2007super}.
% Furthermore, in multi-layer stacking, the predictions from stacker models of each layer serve as inputs for the stacker models of the next layer, creating a hierarchical structure.
Furthermore, in multi-layer stacking, predictions from each layer's stacker models become inputs for the next layer, forming a hierarchical structure.
Table~\ref{tab:automl_sys_ens} shows the stacking strategies of the prominent open-source AutoML systems. 
% However, existing AutoML systems select the base models simply: either all models are included, or the best-performing model from each algorithm class/family is used. 
% However, as shown in Table 1, the stacking strategies of existing AutoML systems are typically pre-defined or are left for the user to configure, rather than being automatically optimized.
We identify three common issues in their post-hoc stacking strategies.

{\bf Issue \#1: Existing approaches lack an effective mechanism for selecting base models.}
% Both theoretical analyses and empirical studies have supported a concept that "many-could-be-better-than-all"  when selecting models for ensemble~\cite{martinez2008analysis,zhou2002ensembling}. And it is generally accepted in the literature that diversity can help improve the performance of the ensemble~\cite{lecun2015deep,bian2021does,hansen1990neural}. 
Existing AutoML systems simply either include all models~(\textsc{All}) or use the best-performing model from each algorithm class~(\textsc{Best}).
% However, on one hand, both theoretical analyses and empirical studies have supported the concept that "many-could-be-better-than-all" when selecting models for an ensemble~\cite{martinez2008analysis,zhou2002ensembling}. 
However, both theory and experiments support that ``many-could-be-better-than-all" in ensemble~\cite{martinez2008analysis,zhou2002ensembling}. 
On the other hand, it is generally accepted in the literature that diversity can help improve the performance of the ensemble~\cite{bian2021does,hansen1990neural,ning2025dual}, while selecting only the best subset lacks the assurance of base model diversity.
% , which is detrimental to the ensemble performance~\cite{lecun2015deep,bian2021does,hansen1990neural}.

{\bf Issue \#2: Existing works have not yet fully explored the potential of multi-layer stacking.}
% In the multi-layer stacking of AutoGluon, the same base models are stacked layer by layer with the aim of progressively refining the quality of predictions. Ultimately, by combining the predictions from the higher-level stacker models, it is possible to achieve better performance than a single-layer stacking. However, even in "best\_quality" mode, AutoGluon’s default stacking does not exceed two layers. 
There's a curious trend that while multi-layer stacking is feasible, existing systems restrict themselves to two layers.
For example, AutoGluon's ``best\_quality" mode limits stacking to two layers.
Similarly, LightAutoML reports using two-layer only for multi-class classification, which is the only case where consistent improvements are observed; one-layer stacking is used otherwise.
H2O, meanwhile, only employs single-layer stacking.
Moreover, there is also a risk of overfitting when aggregating base models, and existing systems lack effective solutions.

{\bf Issue \#3: Existing approaches exhibit inflexible stacking strategies and lack optimization for specific tasks.} 
Several parameters will determine the behavior and performance of stacking.
For example, it involves the number of stacking layers and the choice of the blender model.
However, only AutoGluon dynamically selects the number of stacking layers based on validation scores, while other systems fix it. 
% LightAutoML uses two layers for multiclass classification and one layer for other tasks, while H2O consistently uses a single layer for all tasks. 
As for the blender model of the final layer, AutoGluon and LightAutoML adopt fixed models, while H2O provides several choices for users to specify.

% \subsection{Ensemble Pruning}
% The goal of ensemble pruning is to identify a subset of models that performs as well as, or better than,   ensembling all models. \cite{breiman2001random} demonstrated that the error of an ensemble in continuous problems can be expressed as a linear combination of individual accuracy and pairwise diversity terms. Building on this idea, \cite{zhang2006ensemble} formulated ensemble pruning as a quadratic integer programming problem, using this linear combination as the objective for optimization.
% Specifically, this method involves using a matrix to track errors of member classifiers on validation data, where diagonal elements represent individual classifier errors and off-diagonal elements indicate shared errors between classifiers. The goal is to select a subset of classifiers that minimizes the sum of the corresponding submatrix. This is formulated as a quadratic integer programming problem.
% We notice that the objective of ensembling pruning is similar to that of selecting base models for the ensemble.
% However, directly applying this method presents two issues: first, the method is limited to classification models, using the count of simultaneous errors to reflect diversity. Second, it does not consider the trade-off between individual model performance and inter-model diversity.

\begin{table}[t]
    \vspace{0mm}
    \centering
    % \caption{\small Post-hoc Ensemble Strategy of Current AutoML System.}
    
    \fontsize{9}{\baselineskip}\selectfont
    \setlength{\tabcolsep}{1mm} % Set column separation to 1mm
    \begin{tabular}{lccc}
    \toprule
        \textbf{System} & \makecell{\textbf{Base Model}} & \makecell{\textbf{Layers}} &\textbf{Blender Model}\\
    \midrule
        AutoGluon           & \textsc{All}       & Multiple & Ensemble Selection  \\ 
        LightAutoML         & \textsc{Best}                  & One/Two   & Weighted Avaraging  \\ 
        H2O                 & \textsc{All}/\textsc{Best}                  & One   & Linear/LightGBM/... \\
    \bottomrule
    \end{tabular}
    
    \caption{Stacking strategy of current AutoML systems.}
    \label{tab:automl_sys_ens}
    % \vspace{3mm}
\end{table}

%% file: tex/method.tex
\section{Method}
\label{sec:method}
% In this section, we present \sys – our proposed method for post-hoc stacking ensemble optimization.
In this section, we propose our method for post-hoc stacking ensemble optimization~(PSEO).
As shown in Figure~\ref{fig:overview_of_stacking_generation},
it first presents a \textbf{base model subset selection} algorithm with a trade-off between individual model performance and inter-model diversity. 
After that, \sys builds a deep stacking ensemble with \textbf{Dropout} and \textbf{Retain} mechanisms to fully explore the potential of stacking.
Finally, it defines a hyperparameter space for post-hoc ensemble, and uses a Bayesian optimizer to \textbf{search for the optimal ensemble} within it.
% Next, we will introduce each of the three algorithmic components in detail and explain how they address the issues present in existing methods.

% \begin{figure}[t!]
% 	\centering
% 	% \scalebox{0.95}[.95] {
% 	\includegraphics[width=\linewidth]{images/workflow_opt.pdf}
%          % }
%         % \vspace{-2mm}
% 	\caption{Overview of \sys.}
%         \label{fig:overview_of_stacking_generation}
%         % \vspace{-1em}
% \end{figure}

\subsection{Base Model Subset Selection}
\label{sec:base_model_subset_selection}

In the scenario of post-hoc stacking, it is trivial to incorporate all available base models from the pool $\mathbb{M}$ for stacking. 
However, such a strategy lacks scalability. 
% If a large pool of candidate-based models is provided, stacking all candidate models would result in significant overhead during both ensemble training and subsequent inference phases.
When provided with a large candidate model pool, stacking all models incurs significant overhead during both training and inference.
% Therefore, systems implementing this approach typically have a very compact and narrow CASH space, such as AutoGluon.
In contrast, another approach selects only the best-performing models, which does not fully utilize the diversity of $\mathbb{M}$.
% from each model class for stacking. 
% Thus, the ensemble size equals the number of model classes in the CASH space.
% However, choosing models in this way does not fully utilize the diversity of the model pool.
% Therefore, finding the optimal base model subset for stacking is a problem that needs to be addressed.
Therefore, selecting the optimal base model subset for stacking is a meaningful problem worth investigating.

% Our goal is to select an optimal subset from a pool of candidate base models of size $n$, which is generated during the search for a single optimal model, to form an ensemble.
Nevertheless, the problem of selecting an optimal subset from a pool of candidate base models of size $n$ is a variant of the constrained knapsack problem, which is NP-hard and impractical to solve exactly.
To approximate a solution to this problem, we select a subset by focusing on two key factors: ensemble size and the trade-off between individual model performance and diversity among models.

We first use the pair-wise measures
to represent the diversity of two given models. We utilize the definition in previous research~\cite{poduval2024cash} that is supported by a theoretical framework and shows satisfactory
empirical results. 
The diversity metric between two models $(m_i, m_j)$ is:
% \begin{equation}
%     Div(f_{h_i}, f_{h_j}) = \frac{1}{|\mathcal{D}_{val}|}\sum_{(x,y) \in \mathcal{D}_{val}} \|(y - m_i(x))\odot(y - m_j(x))\|_1
% \end{equation}
\begin{equation}
    \label{eq:div}
    D(m_i, m_j) = \mathbb{E}_{(x,y) \sim \mathcal{D}_{val}} (y - m_i(x))\cdot(y - m_j(x))
\end{equation}
where for classification tasks, $m_i(x)$ is the predictive probability distribution and $y$ is the one-hot label.
For regression tasks, $m_i(x)$ is the predictive value and $y$ is the label.
% The metric emphasizes on minimizing the covariance between errors~(assuming $\mathbb{E}_{(x,y)\sim \mathcal{D}_{val}} [y-m(x)] = 0$).

%  \in \mathbb{R}^{n \times n}
After that, we can build an error covariance matrix $G$, where $G_{ij} = D(m_i, m_j)$.
% $G$ has the interesting properties that 
The diagonal entries $G_{ii}$ represent the mean square errors of model $i$ on the validation data, while the off-diagonal entries $G_{ij}$ quantify the error consistency between model $i$ and $j$. Intuitively, a smaller diagonal entry indicates a better individual model, while a smaller off-diagonal entry suggests a larger error discrepancy between two models, implying greater diversity.
Therefore, it's straightforward that a small value of $\sum_{ij}G_{ij}$ implies a good ensemble.
Furthermore, to trade off between individual model performance and inter-model diversity, we introduce a weighting coefficient $\omega$ to $G$, where the weights of diversity and individual performance are $\omega$ and $1-\omega$:
% After applying the weights, we obtain:
\begin{equation}
    \tilde{G} =  (1 - \omega) \cdot \text{diag}(G) + \omega \cdot \left( G - \text{diag}(G) \right)
\end{equation}
Finally, given $n$ candidate base models, selecting an optimal subset of size $n^{\prime} (n^{\prime} \leq n )$ can now be formulated as:
% with the weight for diversity set to $\omega$ 
\begin{equation}
\begin{aligned}
    &\min_{z} \ z^T \tilde{G} z \quad \text{s.t.} \quad \sum z_i = n^{\prime}, \quad z_i \in \{0, 1\}
\end{aligned}
\end{equation}
This binary quadratic programming (BQP) problem is NP-hard. To enable efficient approximation, we adopt a semidefinite programming (SDP) relaxation~\cite{zhang2006ensemble}. Specifically, we lift the binary variable into a positive semidefinite matrix $Z = zz^\top$, and relax the rank-one and integrality constraints. The relaxed problem becomes:
\begin{equation}
\begin{aligned}
\min_{Z} \ & \mathrm{Tr}(\tilde{G} Z) \quad
\text{s.t.} \ \ \mathrm{Tr}(Z) = n', \ Z \succeq 0, \ Z \succeq zz^T
\end{aligned}
\label{eq:sdp-relax}
\end{equation}
% \begin{equation}
%     Div(\mathbb{M}_{sub}) = \sum_{
%     \substack{m_i, m_j\in\mathbb{M}_{sub} \\  i \neq j}} Div(m_i, m_j)
% \end{equation}
% Equation~\ref {eq:sdp-relax} can be solved by standard semi-definite programming solvers. After that, a feasible binary solution can be recovered. Specifically, we select the indices corresponding to the $n'$ largest values in $\mathrm{diag}(Z)$, and set their corresponding entries in $z$ to 1, with the rest set to 0. When $z_i = 1$, the $i$th model will be selected as one of the base models for stacking.
Equation~\ref {eq:sdp-relax} can be solved by standard semi-definite programming solvers. After that, 
% a feasible binary solution can be recovered. Specifically,
we select the indices corresponding to the $n'$ largest values in $\mathrm{diag}(Z)$ and set their entries in $z$ to 1, the rest to 0. When $z_i = 1$, the $i$-th model will be selected as one of the base models for ensemble.

\begin{figure}[t!]
	\centering
	% \scalebox{0.95}[.95] {
	\includegraphics[width=\linewidth]{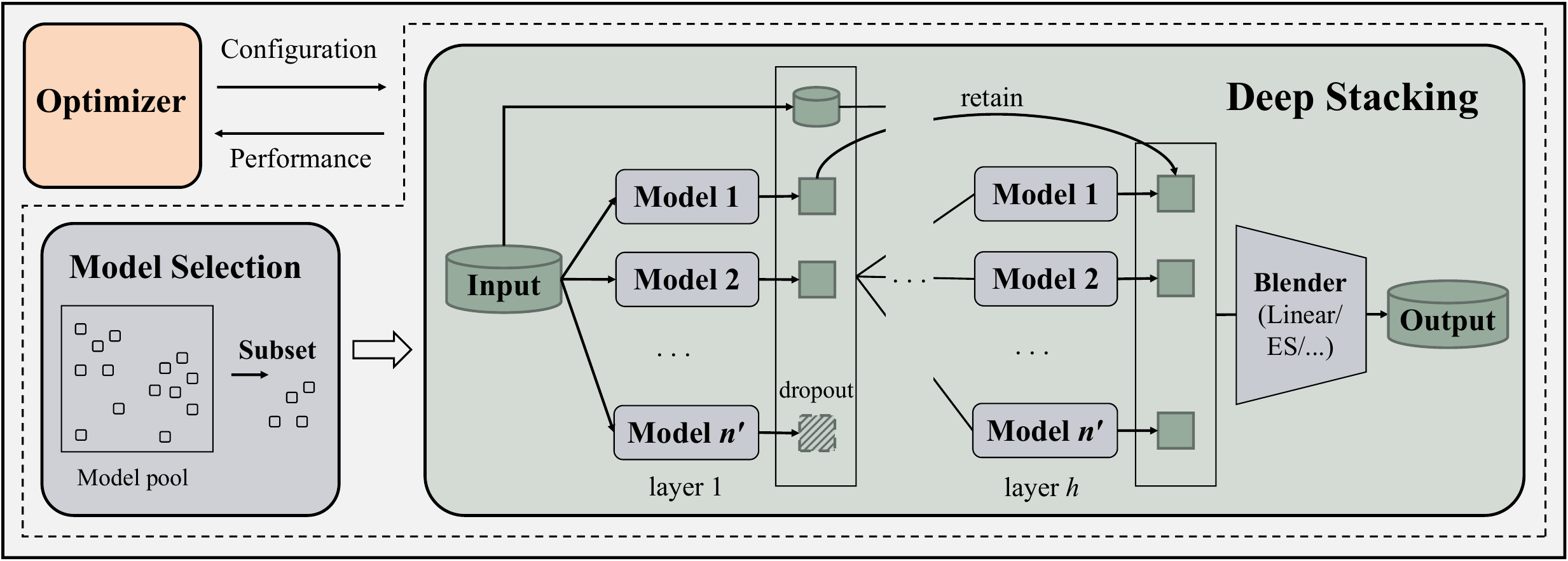}
         % }
        % \vspace{-2mm}
	\caption{Overview of \sys.}
        \label{fig:overview_of_stacking_generation}
        % \vspace{-1em}
\end{figure}

\subsection{Deep Stacking Ensemble}
\label{sec:deep_stacking_ensemble}

After selecting the base model subset, we build a multi-layer stacking ensemble like the right part of Figure~\ref{fig:overview_of_stacking_generation}.
% with a blender model to aggregate the predictive features of the last layer for output. 
% like the right part of Figure~\ref{fig:overview_of_stacking_generation}.
% The original dataset serves as the input to the first layer, whereby 
% where the outputs of the base models in the first layer are input into the next layer as predictive features, and the same process applies to all subsequent layers. In the end, a blender model will aggregate the predictive features of the last layer to generate the output. 
% Like AutoGluon, we reuse all the base models as stacker models for every layer. And stacker models take as input not only the predictions of the previous layer, but also the original data features themselves. 
Like AutoGluon, we reuse all base models as stacker models for each layer, which take inputs from the previous layer's predictive features and the original dataset.
% During the fitting process, $k$-fold cross-validation is applied to the training data to generate predictions at each layer. In each fold, the model is trained on a subset of the training data and generates predictions for the validation data. This way, each layer's stacker model is trained upon lower-layer out-of-fold (OOF) predictions, allowing the use of the entire training set for model training.
During training, we apply $k$-fold cross-validation to generate predictions for each stacker model, 
% In each fold, the model is trained on subsets and predicts on the rest data.
where each fold trains on part of the training set and predicts on the rest.
This ensures stacker models always use out-of-fold (OOF) predictions and allow full training set utilization.
During prediction, we average the outputs of $k$ cross-validation models for each stacker to generate predictions on the validation or test set. 

The above framework, referred to as deep stacking, can be viewed as a variant of a deep learning network, where each stacker model resembles a neuron.
% within a neural network layer.
However, as the stacking depth increases, it becomes more complex, which can lead to two issues: \textbf{overfitting} and \textbf{feature degradation}. 
Overfitting refers to the phenomenon where stacking performs well on the training data but fails to generalize to test data. 
Feature degradation, on the other hand, describes the progressive decline in the quality of predictive features across layers.
% Such degradation can stem from early stackers producing poor or overfitted predictions on the test set, thereby limiting the capacity of subsequent layers to extract useful information, reflecting the classic “garbage in, garbage out” effect.
% Since these issues are not addressed through structural or training methods, existing AutoML systems like AutoGluon and LightAutoML typically do not stack beyond two layers by default. 
As a result, we introduce two mechanisms, Dropout and Retain, to mitigate the aforementioned issues, aiming to fully exploit the potential of deep stacking.
The pseudo-code of the two mechanisms is provided in Appendix~\ref{app:pseudcode_of_training}.

\textbf{Dropout}. 
Consider a stacker model $S_{l,i}$ at layer $l$. % from  $n^{\prime}$ stacked models 
% If a stacker model $S_{l-1, j}$ from the previous layer performs exceptionally well on the training set, there is a high probability that the training of $S_{l, i}$ will overly rely on the output of $S_{l-1, j}$, neglecting other diverse features from the previous layer. For instance, a weighted model might assign a weight close to 1 to $S_{l-1, j}$ while setting the weights of other features close to 0.
If the predictive feature from the $j$-th stacker model $S_{l-1,j}$ of the previous layer is very close to the training set label, $S_{l-1,j}$ may become a dominating model in the training of $S_{l,i}$.
It means there is a high probability that $S_{l, i}$ will overly rely on this predictive feature while neglecting other diverse features. For instance, a weighted model might assign a weight close to 1 to the $j$-th feature while setting others close to 0. 
% If $S_{l-1, j}$'s ability does not generalize well to the test set, its poor predictions can propagate to $S_{l, i}$, 
% potentially continuing to propagate upwards and affecting the final output.
Although it might effectively decrease the training loss, it may not generalize to test samples, leading to overfitting.
% Intuitively, we need to force every stacker to rely on different stacker models from the previous layer to perform the predictions, rather than merely using the preferred one(s).

% Therefore, motivated by the dropout in neural networks\cite{srivastava2014dropout}, we introduce a dropout mechanism in deep stacking. Specifically, for every sample in the training set, the predictive features are the concatenation of all stacker model outputs from the previous layer is $\mathbf{p} = [\mathbf{p}_1, \mathbf{p}_2, \ldots, \mathbf{p}_{n^{\prime}}]$. We first build a mask vector $\mathbf{m} \in \mathbb{R}^{n^{\prime}}$, where $m_i \sim \text{Bernoulli}(1-\gamma)$ with parameter $\gamma$ representing the dropout rate.
% After that, we will drop out some stacker model outputs from the previous layer to get $\tilde{\mathbf{p}}$, where:
% \begin{equation}
%     \tilde{\mathbf{p}}_i = m_i \mathbf{p}_i + (1-m_i)\mathbf{\phi}
% \end{equation}
% The symbol $\phi$ represents the fill vector for the dropped out predictive features. We set it as the average of all the predictive features that were not dropped out:
% \begin{equation}
%     \phi = \frac{1}{\sum m_i} \sum_{i=1}^{n^{\prime}} m_i\mathbf{p}_i
% \end{equation}

% By applying dropout to the predictive features of the previous layer for each sample, we can reduce the performance variance of these features in the training data. This forces the stacker model to learn a more robust ensemble.

Therefore, motivated by the dropout in neural networks\cite{srivastava2014dropout}, we introduce the Dropout mechanisms in deep stacking.
The key insight of Dropout is to encourage the ensemble to depend on diverse base models, preventing reliance on any dominating models that show strong training set performance.
Specifically, receiving the predictive features from the previous layer, we first build a dropout mask for each predictive feature as: 
$$
    m_i \sim \text{Bernoulli}(\frac{\mathcal{L}_{\text{min}}}{\mathcal{L}_i}\gamma_0)
$$
where $\gamma_0$ is the only parameter representing the dropout rate, $\mathcal{L}_{\text{min}}$ represents the minimum loss with the training label among all predictive features from the previous layer, while $\mathcal{L}_{i}$ denotes the loss of the $i$-th predictive feature.
When $m_i$ is 1, we exclude the $i$-th predictive feature and train the current stacker model using the retained features. 
Additionally, due to the randomness introduced by the dropout process, we perform multiple repetitions to ensure robustness.
For each stacker model,  we apply dropout independently to each fold of cross-validation. 
For the blender model, we repeatedly apply dropout and training, and compute an average of multiple blender models' outputs for predicting.
To sum up, by assigning higher dropout probabilities to features closer to the labels on the training set, we offer the opportunity to reduce the dominance of any single model in the ensemble. 

\begin{theorem}
    \label{the:decrease}
    Consider a set of uncorrelated base models predictions $\mathcal{Z} = \{z_1, z_2, ..., z_{n'} \}$ in a weighted ensemble $z_{\text{ens}} = \sum_i \beta_iz_i$.
    The $i$-th prediction is dropout with rate $d_i$($d_i = \rho_i\gamma_0$ and $\rho_1 = 1 \geq \rho_i \geq 0, i\in\{1,2,\ldots,n'\}$). 
    If we perform $N$ independent random samplings of $\mathcal{Z}$ and get the average weight of each prediction,  
    as $N \to \infty$, the proportion of $z_1$'s absolute average weight in the sum of all predictions' (i.e. $|\tilde{\beta}_1| / \sum_{i=1}^{n'}|\tilde{\beta}_i|$) decreases as $\gamma_0$ increases.
\end{theorem}
We give a proof of Theorem~\ref{the:decrease} in Appendix~\ref{app:proof}, which shows that with sufficient sampling, averaging results from multiple training with Dropout can reduce the weight proportion of the predictive feature with the highest dropout probability in the weighted ensemble.
In our scenario, the contribution of the feature with the lowest training loss will be reduced.
Although the feature with the lowest training loss may not always be the dominating model, they are aligned with high probability. Even if not, as shown in Section~\ref{sec:exp_of_dropout_and_retain}, Dropout may still reduce the weight of the dominating model, as it tends to have lower training loss and thus a higher dropout probability compared to others.
% This indicates that the Dropout mechanism offers an opportunity to reduce the dominance of models with strong performance on the training set, forcing the ensemble to utilize more base models and thereby alleviating overfitting.
% In addition, although some stacking models or blender models naturally possess methods to mitigate overfitting (e.g., linear models), we provide a general mechanism from the structural level of the stacking framework to alleviate overfitting. This can complement existing methods and play a more critical role for models that lack such methods (e.g., KNN).

\textbf{Retain}. 
% Consider the following scenario: if a stacker model $S_{l-1,j}$ suffers from severe overfitting, and a subsequent model $S_{l,i}$ fails to effectively handle this issue and instead relies heavily on the predictive features generated by $S_{l-1,j}$, the overfitting may propagate forward. Alternatively, if $S_{l-1,j}$ is poorly trained and produces low-quality predictive features, the subsequent stacker models may be adversely affected by these degraded features, leading to insufficient training. Both cases can result in a progressive deterioration of predictive features across layers when evaluated on out-of-sample data (e.g., the test set), ultimately compromising the final output. We refer to this phenomenon as feature degradation in deep stacking.
% Consider the following scenario:
When a stacker model $S_{l,i}$ generates low-quality predictive features on out-of-sample data—due to overfitting or poor training—and subsequent models fail to correct or compensate for these degraded features, errors can propagate across layers. This leads to a progressive decline in out-of-sample performance across layers, ultimately compromising the final output. 
We refer to this phenomenon as feature degradation in deep stacking.

To address this issue, we introduce the Retain mechanism.  The insight of Retain to prevent the continued propagation of feature degradation after it occurs at a certain stacker model.
Specifically, if Retain is enabled, it compares each stacker model’s predictive loss on the validation set with that of its counterpart from the previous layer~(i.e., the stacker model at the same position and with the same configuration).
If the current model underperforms, its output is replaced with that generated by its predecessor; otherwise, its output is retained.
In this way, we offer the potential for the predictive features to progressively converge toward better representations across layers on data outside the training set.

% To highlight the effect of predictive feature quality on ensemble performance, Figure~\ref{fig:avg_rank_of_output} uses a linear model as the blender and shows the average test ranking of the final stacking ensemble with and without the retain mechanism, as stacking depth increases. The results indicate that the retain mechanism performs increasingly better, widening the gap over the baseline.
% As a result, by using validation performance to decide whether to retain the output of a stacker model, we offer the potential for the predictive features to progressively converge toward better representations across layers on data outside the training set.

% \textbf{Training with Dropout and Retain}. 

\subsection{Post-hoc Stacking Ensemble Optimization}
Overall, post-hoc stacking can be divided into two stages: selecting a subset of base models and constructing a specific stacking ensemble based on the subset.
Building upon the preceding design, we notice that several hyperparameters will determine the behavior of post-hoc stacking:
1) ensemble size and 2) diversity weight determine the size of the selected base model subset and the trade-off between individual performance and diversity within the subset.
During the stacking stage, 3) the number of stacking layers and 4) the choice of blender model define the stacking structure; while 5) ``retain or not" and 6) dropout rate control the training mechanisms to overcome the issues of overfitting and feature degradation. 
Therefore, we define the hyperparameter space of post-hoc stacking optimization as the combination of the above six hyperparameters.
% We can define the hyperparameter list of stacking optimization as follows:
% \begin{equation}
% \begin{array}{rl}
% \mathcal{X} = \left\{
% \begin{array}{l}
% \text{ensemble size} \\
% \text{diversity weight} \\
% \text{number of layers} \\
% \text{blender model} \\
% \text{dropout rate} \\
% \text{retain or not}
% \end{array}
% \right.
% &
% \begin{tikzpicture}[baseline=(current bounding box.center), xshift=-10.5cm]
%     \node (a) at (-1.2,1.1) {};
%     \node (b) at (-1.2,0.7) {};
%     \node (c) at (-1.2,0.3) {};
%     \node (d) at (-1.2,-0.1) {};
%     \node (e) at (-1.2,-0.5) {};
%     \node (f) at (-1.2,-1.1) {};
%     % 第一组：1-2行
%     \draw [decorate,decoration={brace,amplitude=6pt}] (a.north) -- (b.south) node [midway,right,xshift=6pt] {subset selection};
%     % 第二组：3-5行
%     \draw [decorate,decoration={brace,amplitude=6pt}] (c.north) -- (f.north) node [midway,right,xshift=6pt] {deep stacking};
% \end{tikzpicture}
% \end{array}
% % \vspace{-0.5em}
% \end{equation}
% The maximum number of layers and the maximum ensemble size can be configured by the user based on their computational resources. \sys includes a set of commonly used blender models in the search space. However, additional primitives can be added by the user as needed.
In order to identify the optimal post-hoc stacking strategy over the hyperparameter space, we employ the state-of-the-art black-box optimization technique, Bayesian Optimization (BO)~\cite{snoek2012practical,hutter2011sequential}.
First, \sys will collect the predictions of all models in the candidate pool on the validation set. 
% Then at each iteration, it employs Bayesian Optimization to suggest a new ensemble configuration by maximizing the acquisition function over the hyperparameter space. Based on the selected configuration, it constructs a stacking ensemble model on the training set and evaluates its performance on the validation set.
At each iteration, it uses BO to suggest a new ensemble configuration. 
It then constructs a corresponding post-hoc ensemble on the training set and evaluates its performance on the validation set.
After exhausting the budget, the best ensemble configuration found in the observations is returned.
Appendix~\ref{app:pseudcode_of_sys} shows the pseudo-code of \sys and further details on the disk storage and reuse strategies to reduce computation.

% Moreover, before BO gets sufficient observations to train
% the surrogate model, we sample several configurations for initialization. with the Latin Hypercube Sampling (LHS)~\cite{mckay2000comparison}.

\subsection{Discussions}

% In this section, we provide a discussion on \sys as follows:

\textbf{Extension.} 
As a post-hoc ensemble framework, \sys operates independently of the tuning algorithm used in the single best model searching stage, and can be applied to any AutoML system that produces a candidate model pool. % or an optimization trajectory.
% Moreover, the term model in our context refers to a generalized concept, which can represent a complete machine learning pipeline, including preprocessing, feature engineering, and the learning algorithm itself.
% \sys is also independent of the task type and metrics.
\sys also provides a general ensemble framework, within which other ensemble methods (e.g., ensemble selection) can be integrated as options of blender models.

% \textbf{Time Complexity.}
% The time complexity of each iteration to suggest a new ensemble configuration is $\mathcal{O}(|D|log|D|)$, which is dominated by the cost of fitting a probabilistic random forest surrogate.
% Through solving the SDP problem with the "Clarabel" solver, the time complexity to select the base model subset is $\mathcal{O}(|\mathbb{M}|^2)$.

\textbf{Overhead Analysis.}
Suppose $T_f$ is the average cost of fitting a stacker model on one fold, $T_{p_1}$ and $T_{p_2}$ are the average cost of a stacker model to predict on one fold of the training set and the validation set, respectively.
During each iteration, the computation cost of training and validating a stacking is $khn'(T_f+T_{p_1}+T_{p_2})$, where $k$ is the fold number of cross-validation, $h$ is the number of layers, and $n'$ is the ensemble size.
In practice, the training and validation of $kn'$ stacker models in each layer can be executed in parallel, which can significantly reduces the computational cost. 
% Moreover, by storing previously constructed stacking ensembles on disk, \sys enables partial reuse, thereby further reducing the overhead.
% In scenarios where computational resources are limited or the dataset is large, users may reduce the number of stacking layers and the upper bound of ensemble size. In more extreme cases, the stacking can be restricted to a single layer, and cross-validation can be omitted by training the blender model directly on the validation set, thereby avoiding any additional training overhead of stacker models.

\textbf{Difference with Previous Methods.}
Methods such as Auto-WEKA~\cite{thornton2013auto} and Autostacker~\cite{chen2018autostacker} also conduct a searching of ensemble.
We discuss the difference between \sys and previous methods as follows:
\textbf{D1}. The first difference lies in the search space.
Auto-WEKA and Autostacker adopt an expanded search space covering both ensemble strategies and base model configurations.
In contrast, \sys adopts the post-hoc ensemble, where the search space covers base model selection and ensemble methods.
\textbf{D2}. The main difference is in how the ensemble search space is explored.
% Auto-WEKA and Autostacker directly search over the space whose dimensionality is proportional to the maximum ensemble size, causing an exponential growth in configurations. 
Auto-WEKA and Autostacker perform a joint search over both base models and their ensemble, leading to a combinatorial search space whose size grows exponentially with the maximum ensemble size.
In contrast, \sys represents base model subset selection using two key hyperparameters, allowing for only a linear increase in the number of candidate configurations.
% making the dimensionality independent of the maximum ensemble size, and leading to a linear increase in the number of candidate configurations.
% Methods such as Auto-WEKA~\cite{thornton2013auto} and Autostacker~\cite{chen2018autostacker} also adopt an ensemble learning perspective.
% The key distinction between \sys and these methods lies in how the ensemble is constructed. 
% These approaches adopt the one-stage ensemble, directly searching for an ensemble over an expanded search space that includes both the ensemble strategy and the hyperparameters of each base model. 
% As a result, if the maximum number of base models is $N$, the dimensionality of the search space becomes $N$ times that of a single model’s hyperparameter space, and the number of possible configurations grows exponentially.
% In contrast, \sys adopts the post-hoc ensemble, building a post-hoc stacking ensemble search space and searching for the best stacking ensemble based on a pool of base models obtained from prior single-model optimization.
% In this way, the dimensionality of the ensemble search space is independent of the ensemble size, and the total number of ensemble configurations increases only linearly with $N$.

%% file: tex/experiment.tex
\section{Experiment}
\label{sec:experiment}
\subsection{Experiment Setup}
\label{sec:experiment_setup}

\begin{table}[t]
    % \vspace{0mm}
    \centering
    
    \begin{tabular}{l|cc}
    \toprule
        \textbf{Hyperparameter} & \textbf{Type} &\textbf{Range}\\
    \midrule
         Ensemble size & Int & [5, 50] \\
         Diversity weight & Float & [0, 0.5] \\
        \hline
        Number of layers & Int & [1, 5] \\
         Blender model & Categorical & \{ES, Linear, LGB\} \\
         Dropout rate & Float & [0, 0.4] \\
         Retain or not & Categorical & \{True, False\} \\
    \bottomrule
    \end{tabular}
    \caption{Ensemble hyperparameter space of \sys.}
    \label{tab:search_space}
    % \vspace{-1em}
\end{table}

\textbf{Baselines.}
We compare the proposed \sys with 15 baselines. 
1) The best model according to the validation metric~(\textbf{Single-Best}); 
--- \textit{Three one-step ensemble learning methods}: 
2) Ensemble optimization (\textbf{EO})~\cite{levesque2016bayesian} that maintains and updates an ensemble pool greedily; 
3) \textbf{Autostacker}~\cite{chen2018autostacker} that uses an evolutionary strategy to jointly search for the stacking structure and base model hyperparameters; 
4) \textbf{OptDivBO}~\cite{poduval2024cash} that maintains an ensemble selection and considers diversity while selecting the next model to train;
--- \textit{Two post-hoc ensemble designs based on ensemble selection}: 
5) Ensemble selection~(\textbf{ES})~\cite{caruana2004ensemble};
6) \textbf{CMA-ES}~\cite{purucker2023cma} that uses an evolutionary strategy to search for an optimal ensemble weight;
--- \textit{Nine post-hoc ensemble designs based on stacking}:
7)-14) Based on the post-hoc stacking strategies adopted by existing AutoML systems, we consider three key hyperparameters. 
First, for base model selection, we adopt two commonly used strategies: (a) selecting all available models (\textsc{All}), and (b) selecting the best-performing model from each algorithm class (\textsc{Best}). 
Second, we vary the number of stacking layers, considering both one-layer and two-layer configurations. %, as most existing AutoML systems typically limit stacking depth to no more than two layers by default. 
Third, for the final blender model, we use two widely adopted approaches: ensemble selection (ES) and linear regression (Linear). 
By combining these choices, we obtain eight commonly used stacking strategies in existing AutoML systems~(e.g., \textbf{\textsc{All}-Linear-L2} stacks all base models for two layers, and uses a linear blender for output);
15) We extracted AutoGluon's ``best\_quality" ensemble strategy~(\textbf{AutoGluon-}), which stacks all candidate models with up to two layers. The final number of layers is determined based on validation scores. If the time budget is limited, time is evenly allocated across layers, with models trained sequentially based on their validation scores.

\textbf{Datasets and metrics.}
To compare \sys with baselines, we use 80 real-world ML datasets from the OpenML repository~\cite{vanschoren2014openml}, 50 for classification tasks (CLS) and 30 for regression tasks (REG). The number of samples for CLS ranges from 1k to 110k, and that for REG ranges from 0.6k to 166k.
Further dataset details can be found in Appendix~\ref{app:dataset_details}.
We report the test error for all classification tasks to measure misclassification rates, which equals $1-accuracy$. While for regression tasks, we use $mean\ squared\ error$~(MSE) to assess prediction error.

\textbf{Construction of the base model pool.}
For all post-hoc ensemble methods, we construct the candidate base model pool $\mathbb{M}$ using the AutoML system VolcanoML~\cite{li2023volcanoml}. It defines a CASH search space encompassing algorithmic choices and feature engineering hyperparameters, which contains 100 hyperparameters in total~(see Appendix~\ref{app:search_space_of_cash} for details). We run the Bayesian optimization method of VolcanoML over this space and collect all evaluated models to form $\mathbb{M}$.
In each iteration, models are fitted on the training set and evaluated on the validation set. All the fitted models are saved to disk for ensemble utilization.

\textbf{Implementation Details}.
Each dataset undergoes a split into three distinct sets: training (60\%), validation (20\%), and test (20\%).
Table~\ref{tab:search_space} shows the specific hyperparameter space of \sys. The value ranges of the parameters can be adjusted based on the available budget or user preferences.
% The number of folds for cross-validation during stacking training is set to 5, and we set the number of repetitions for dropout training of the blender model the same as the number of folds.
We used 5-fold cross-validation for stacking training and performed dropout training of the blender model with the same number of repetitions.
More implementation details for other baselines and the training process are provided in Appendix~\ref{app:implementation_details}.

For all post-hoc ensemble approaches, 3600 seconds are dedicated to constructing the candidate base model pool $\mathbb{M}$, yielding an average of 437 base models per task.
For the baseline using a fixed post-hoc ensemble strategy, we execute it to completion. %  and record the total time cost
For methods that perform additional optimization during the post-hoc ensemble phase~(CMA-ES, AutoGluon-, and \sys), we allocate an extra 3600 seconds for ensemble-level optimization.
For the three ensemble learning methods, we impose a total budget of 7200 seconds, due to the adoption of a one-step optimization strategy rather than a two-step post-hoc ensemble.

\subsection{Effectiveness of Base Model Subset Selection}
\label{exp:effectiveness_of_base_model_selection}
We first conduct an experiment to validate the effectiveness of \sys's base model subset selection algorithm.
To investigate the impact of different ensemble size $n'$ and diversity weight $\omega$ on the performance of stacking, we conduct a grid search over their feasible ranges. Specifically, $n'$ is varied from 10 to 50 in increments of 10, 
% i.e., $\{10, 20, 30, 40, 50\}$, 
while $\omega$ is varied from 0 to 0.5 with a step size of 0.1.
% i.e., $\{0.0, 0.1, 0.2, 0.3, 0.4, 0.5\}$.
For each combination of $n'$ and $\omega$, we generate a corresponding subset of base models and apply a one-layer stacking with Linear as the blender model. 
Then, the optimal combination~(\textsc{Opt}) is selected based on validation performance and evaluated on the test set.
We compare its test errors with those of 30 fixed hyperparameter combinations. 
Furthermore, we include comparisons with strategies commonly employed in AutoML systems, namely: \textsc{All} and \textsc{Best}.
% selecting all available models (\textsc{All}), and selecting the best-performing model from each algorithm class (\textsc{Best}).
Figure~\ref{fig:size_ratio_rank} presents the average test ranks of stacking ensembles constructed using base models selected by the different strategies. Each region represents a method, with color intensity and the number indicating the average rank (lighter is better).
The $5 \times 6$ grid in the bottom-left corresponds to a grid search over $n'$ and $\omega$. The surrounding L-shaped region marks the optimal combinations selected on the validation set~(\textsc{Opt}). The two blocks on the far right represent the \textsc{Best} and \textsc{All} strategies.

We derive three observations from the results:
1) Ensembles based on low-diversity models show inferior performance. For instance, all the $\omega=0$ strategies and \textsc{Best} yield average ranks between 17.8 and 20.4.
2) \sys's base model selection mostly outperforms \textsc{All} and \textsc{Best}, with 18 and 29 wins out of 30 configurations, respectively.
% Moreover, because the average ensemble size of \textsc{All} is 437, even the largest ensemble of \sys achieves better performance with only 1/9 of the stacking training time.
3) Tuning $n'$ and $\omega$ leads to significant improvements,  
% This is evidenced by the fact that \textsc{Opt} achieves a rank of 8.8 vs.\ 12.3 for the second-best baseline.
 with \textsc{Opt} achieving a rank of 8.8 vs.\ 12.3 for the second-best baseline.
The superiority of \textsc{Opt} stems from two factors: 1) it balances individual model performance and inter-model diversity with diversity weights, as further demonstrated in Appendix~\ref{app:exp_trade-off}. 2) it optimizes the ensemble size and diversity weight to achieve a task-specific base model selection strategy.

\begin{figure}[t!]
	\centering
	% \scalebox{0.95}[.95] {
	\includegraphics[width=0.9\linewidth]{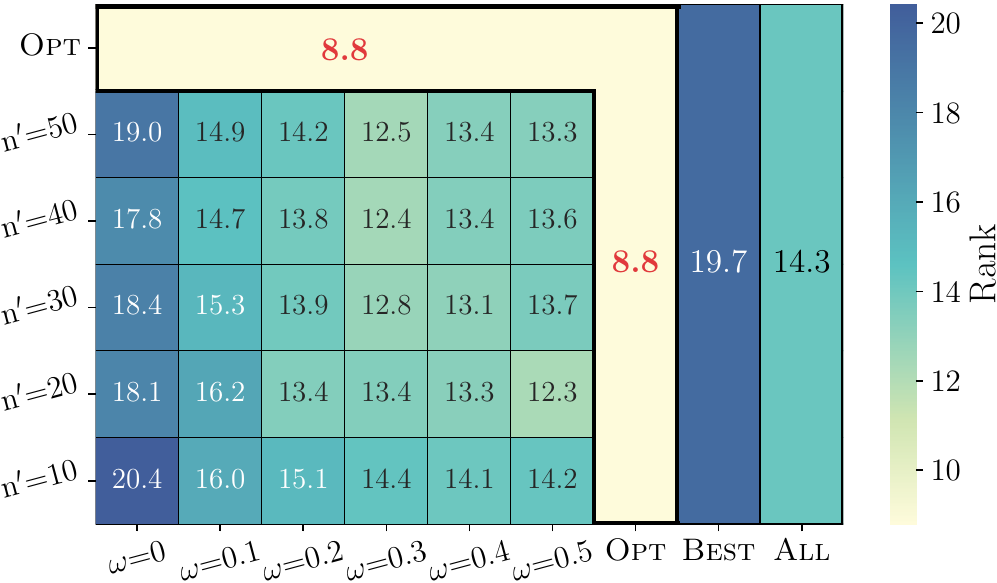}
         % }
        % \vspace{-2em}
	\caption{Average test ranks of stacking with different base model subset selection methods across 80 datasets.}
        \label{fig:size_ratio_rank}
        % \vspace{-1em}
\end{figure}

\begin{figure}[tb]
% % \vspace{-3mm}
\centering

\subfigure[Proportion of max weight.]{
\includegraphics[width=0.47\linewidth]{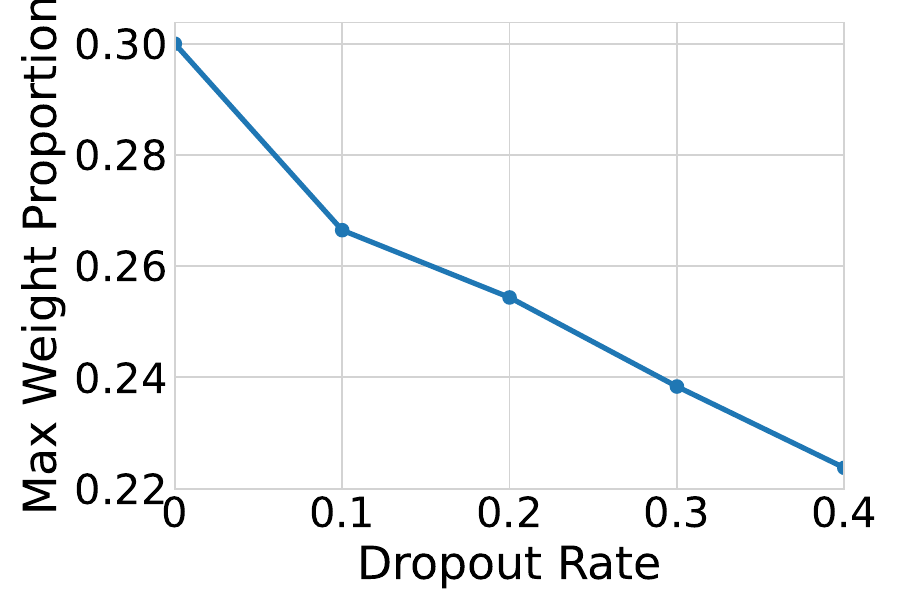}
\label{fig:maximum_weight}
}
\subfigure[Improvement of stackers.]{
\includegraphics[width=0.475\linewidth]{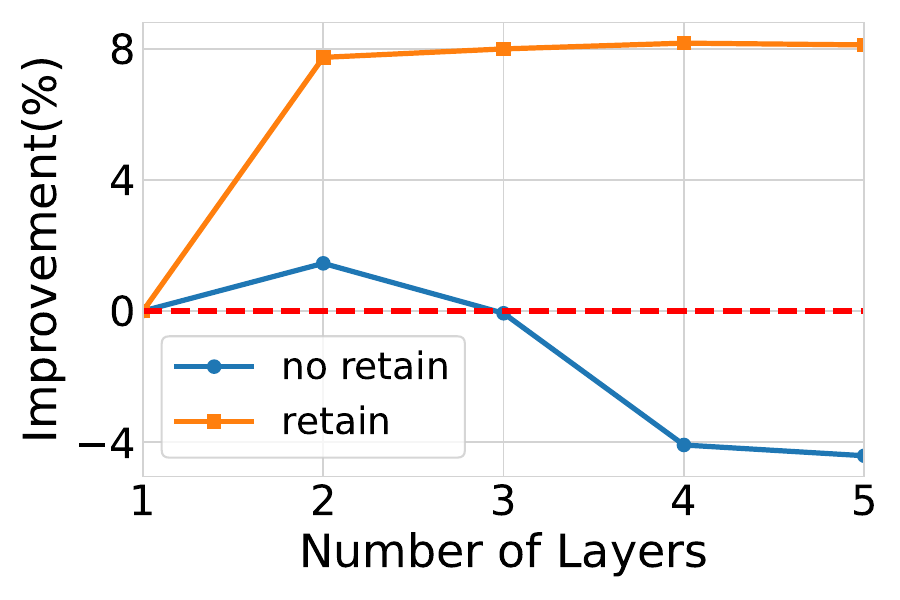}
\label{fig:improvement_of_stackers}
}
% \subfigure[Average rank of output.]{
% \label{fig:avg_rank_of_output}
% \includegraphics[width=0.46\linewidth]{images/rank_with_layers.pdf}
% }

% \vspace{-3mm}
% \caption{\small Performance of stackers and outputs with(out) retain.}
\caption{Effectiveness of Dropout and Retain mechanisms.}
\label{fig:effectiveness_of_dropout_and_retain}
% \vspace{-1em}
\end{figure}

\subsection{Effectiveness of Dropout and Retain}
\label{sec:exp_of_dropout_and_retain}
% To demonstrate the flexibility and potential of \sys's stacking framework, 
This section evaluates the impact of Dropout and Retain mechanisms. Here, we select 30 base models with a diversity weight of 0.3 and use ES as the blender model.

In Figure~\ref{fig:maximum_weight}, for each dropout rate $\gamma_0$ ranging from 0 to 0.4, we trained the ES five times, averaged the weights, and reported the average proportion of the largest base model weight across 80 tasks.
We can observe that as $\gamma_0$ increases, the largest weight shows a decreasing trend.
Appendix~\ref{app:exp_effectiveness_of_dropout} further shows that as the dropout rate increases, the gap between the training and test errors gradually narrows, indicating a reduction in overfitting.
In Figure~\ref{fig:improvement_of_stackers}, we compute the average test error of stacker models’ predictive features at each layer, and report each layer’s average improvement rate relative to the first layer across 80 tasks.
It is observed that without the Retain mechanism, feature quality peaks at the second layer and subsequently deteriorates.
% stacking two layers leads to an average improvement in predictive features. However, beyond the second layer, feature quality drops below that of the first layer and continues to degrade.
In contrast, with the Retain mechanism, the predictive features show a more significant improvement, and no obvious degradation is observed as the number of layers increases.
Appendix~\ref{app:exp_effectiveness_of_retain} provides further analysis of how predictive feature quality affects ensemble performance, concluding that in higher-level stacking, Retain significantly benefits the ensemble.

\begin{table*}[t]
    % \vspace{0mm}
    \centering
    % \fontsize{9}{11}\selectfont
    \setlength{\tabcolsep}{1.8mm} % Set column separation to 1mm
    % \caption{Average rank and average total cost across different datasets.}

    \begin{subtable}
        \centering
        \begin{tabular}{lccccccccc}
        \toprule
             & Single-Best & EO & Autostacker & OptDivBO & ES & CMA-ES & \textsc{All}-ES-L1 & \textsc{All}-ES-L2 & \textsc{All}-Linear-L1 \\
        \midrule
% CLS & 9.32 & 9.02 & 10.76 & 7.42 & 7.20 & 7.98 & 7.30 & 5.80 & \underline{5.68} \\
% REG & 13.03 & 12.57 & 10.07 & 7.43 & 8.53 & 9.17 & 8.13 & 6.93 & 7.50 \\
% ALL & 10.71 & 10.35 & 10.50 & 7.42 & 7.70 & 8.43 & 7.61 & 6.22 & 6.36 \\
CLS & 9.44 & 9.12 & 9.46 & 7.10 & 7.36 & 8.14 & 7.44 & 5.86 & \underline{5.72} \\
REG & 13.03 & 12.43 & 10.07 & 8.33 & 8.47 & 9.17 & 8.07 & 6.90 & 7.43 \\
ALL & 10.79 & 10.36 & 9.69 & 7.56 & 7.78 & 8.53 & 7.67 & 6.25 & 6.36 \\
        % \hline
        % \addlinespace[1mm]
        % Cost & 3600 & 7200 & 7200 & 7200 & 3601 & 7200 & 7606 & 11972 & 7683 \\
        \bottomrule
        \end{tabular}
    \end{subtable}

    % \vspace{1mm} % 两个表格之间的间距
    
    \begin{subtable}
        \centering
        \begin{tabular}{lccccccc}
        \toprule
              & \textsc{All}-Linear-L2 & \textsc{Best}-ES-L1 & \textsc{Best}-ES-L2 & \textsc{Best}-Linear-L1 & \textsc{Best}-Linear-L2 & AutoGluon- & \sys \\
        \midrule
% CLS & 7.28 & 7.30 & 7.10 & 8.12 & 6.54 & 6.06 & \textbf{2.58} \\
% REG & 7.30 & 9.80 & 7.77 & 9.43 & 7.50 & \underline{6.20} & \textbf{3.67} \\
% ALL & 7.29 & 8.24 & 7.35 & 8.61 & 6.90 & \underline{6.11} & \textbf{2.99} \\
CLS & 7.36 & 7.50 & 7.28 & 8.26 & 6.74 & 6.18 & \textbf{2.56} \\
REG & 7.23 & 9.73 & 7.80 & 9.40 & 7.43 & \underline{6.20} & \textbf{3.63} \\
ALL & 7.31 & 8.34 & 7.47 & 8.69 & 7.00 & \underline{6.19} & \textbf{2.96} \\
        % Champion  &   6 &    7   &    6   &   5  &   7  & 11& 30 \\
        % \hline
        % \addlinespace[1mm]
        % Cost & 12009 & 3703 & 3837 & 3712 & 3871 & 7200 & 7200 \\
        \bottomrule
        \end{tabular}
    \end{subtable}
  
    \caption{Average test rank across 50 classification tasks (CLS), 30 regression tasks (REG), and all tasks (ALL) for 16 methods. A smaller rank is better. The best methods are shown in bold, while the second-best methods are underlined.}
    % The cost row shows the average total time each method spent solving the CASH tasks, measured in seconds.
    \label{tab:max_exp}
\end{table*}

\subsection{Comparison with Baselines}
\label{sec:main_exp}

This section compares \sys with state-of-the-art baselines on 80 real-world CASH problems.
Table~\ref{tab:max_exp} shows the average test rank across different datasets, from which we can get five observations:
1) Ensemble indeed helps improve the performance of CASH results. 
% The average rank of the single best model selected by the validation set~(Single-Best) is only 10.79, the lowest among all methods.
The average rank of Single-Best is only 10.79, the lowest among all methods.
2) Among the ensemble learning methods, direct search for an ensemble~(Autostacker, EO) is ineffective. 
This is because the huge search space, covering both ensemble and model hyperparameters, is challenging to optimize with Autostacker's evolutionary algorithm or EO's greedy update strategy.
OptDivBO achieves a rank of 7.56, outperforming Autostacker~(9.69) and EO~(10.36), as it introduces ensemble selection to update the ensemble pool, thereby avoiding EO's instability caused by including poor models during optimization.
% ES is a strong baseline, achieving a rank of 7.78 with minimal cost among all ensemble methods. 
% , requiring only 1 extra minute for post-hoc ensemble after 3600 seconds of base model collection
% CMA-ES may lead to overfitting through genetic search, resulting in a lower rank~(8.53) than ES~(7.78).
3) Among all post-hoc ensemble methods, stacking shows great potential. AutoGluon-, \textsc{All}-ES-L2, and \textsc{All}-Linear-L1 achieve ranks between 6.19 and 6.36, higher than ES~(7.78) and CMA-ES~(8.53). 
This is because ES relies on a fixed greedy strategy, which inherently limits its upper bound, and CMA-ES may lead to overfitting. 
4) We observe that stacking all models~(\textsc{All}-*-*) outperforms stacking the best ones~(\textsc{Best}-*-*), as it fully leverages the diversity in $\mathbb{M}$. However, stacking all models incurs substantial costs, with single-layer stacking taking over an hour on average and multi-layer stacking becoming even more unacceptable.
5) Among all methods, \sys significantly outperforms the others. While the second-best baseline achieves a rank of 6.19, the rank of \sys is 2.96. 
% Among the 8 fixed-strategy stacking ensemble baselines, either the cost is too high or the stacking performance is poor.
% AutoGluon-, which optimizes the number of stacking layers, outperforms the 8 fixed-strategy stacking methods and achieves the second rank~(6.19) across 80 datasets.
The specific loss per dataset per method is shown in Appendix~\ref{app:exp_performance_perdataset_permethod}.

We attribute the superiority of \sys to three key factors:
1) As demonstrated in Section~\ref{exp:effectiveness_of_base_model_selection}, it provides a customized base model selection for different tasks.
2) Dropout and Retain mechanisms mitigate the issues of overfitting and feature degradation, enhancing the potential of deep stacking ensembles. We  conduct ablation studies of them in Appendix~\ref{app:exp_ablation}.
3) Through optimization, \sys builds task-specific stacking structures. 
As shown in Appendix~\ref{app:exp_effectiveness_of_struc}, 
optimizing stacking hyperparameters yields significant benefits.
% We analyze the optimal configurations in Appendix~\ref{app:exp_optimal_conf_ana}, aiming to provide insights for designing post-hoc ensembles tailored to different tasks.

\begin{figure}[t!]
	\centering
	% \scalebox{0.95}[.95] {
	\includegraphics[width=\linewidth]{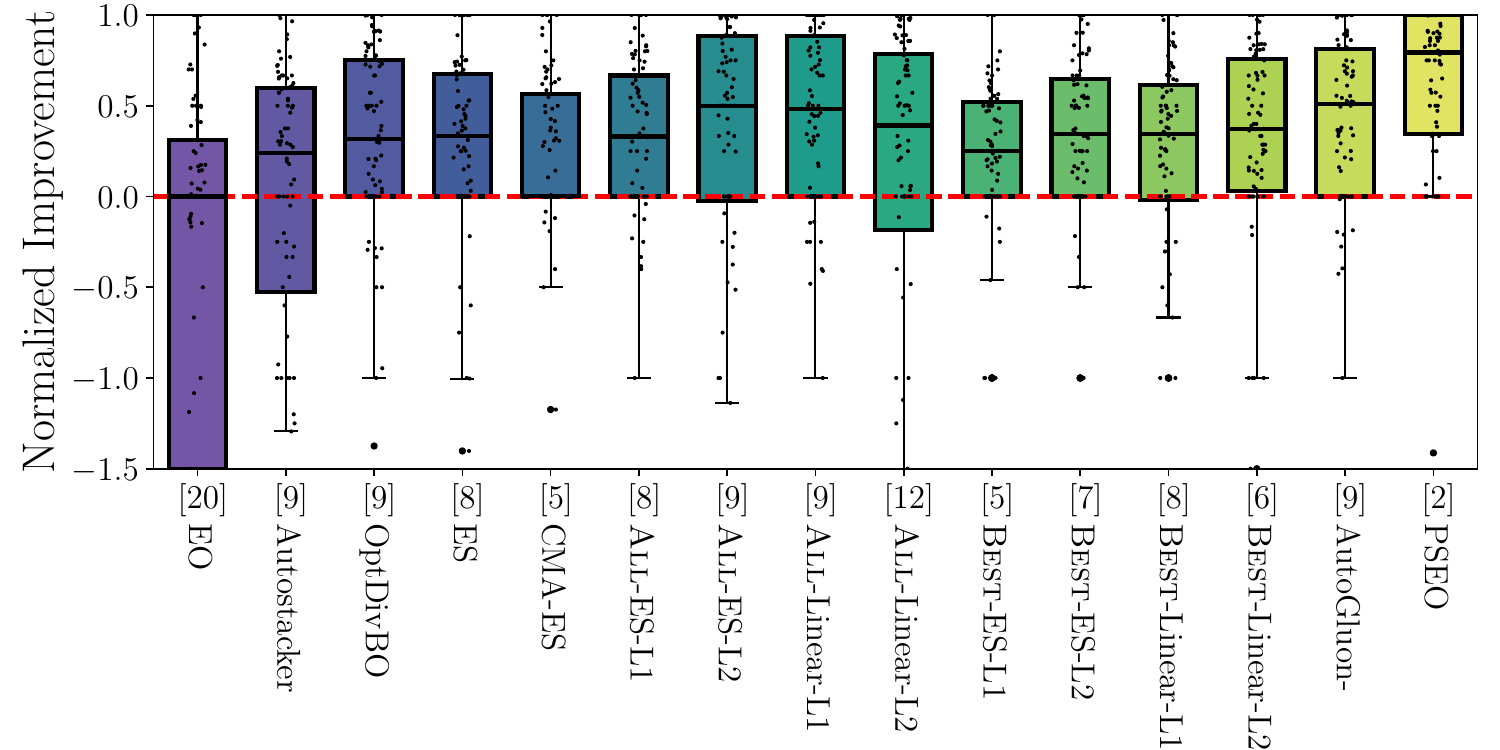}
         % }
        % \vspace{-2em}
	\caption{Normalised Improvement Boxplots: Higher is better. Each dot represents a dataset. The number in square brackets counts the outliers that are not shown in the plot.}
        \label{fig:main_exp}
        % \vspace{-1em}
\end{figure}

\textbf{Normalised Improvement.} 
To further investigate our results, we use boxplots of the normalized improvement to visualize the distribution of relative performance for all methods across 80 datasets in Figure~\ref{fig:main_exp}.
Specifically, for each dataset, 1 corresponds to the minimum test loss achieved among all methods on it, and 0 corresponds to the test loss of Single-Best. 
If Single-Best is the best, we penalize the relative performance of worse methods to -10. % set everything as good as Single-Best to 0 and 
We can conclude that \sys has better relative performance distributions than all the baselines (see the medians and whiskers).

% \textbf{Optimal Configuration Analysis.}
% We aim to explore the distinctive features of the optimal ensemble configurations obtained by \sys across various tasks. Our analysis reveals five key characteristics:
% 1) For base model subset selection, a higher diversity weight is generally more effective. 
% % A large ensemble size combined with a small diversity weight is typically not optimal.
% 2) Regarding the number of stacking layers, two layers are typically the best choice, especially for regression models. Scenarios with more or fewer layers are rare for regression. 
% However, classification models show potential with additional layers.
% 3) For smaller datasets, a larger dropout rate may be necessary, as overfitting is more likely to occur.
% 4) For blender models, ES and Linear are typically more optimal choices than LGB. ES performs particularly well on smaller datasets, while Linear tends to improve as the dataset size increases.
% 5) The Retain mechanism is recommended to be enabled, which is a necessary condition for achieving the best performance in higher-level stacking.
% % When the optimal number of stacking layers exceeds one, the value of retain or not is almost always True.
% Appendix B.6 presents the detailed experimental analysis process.
% We hope this analysis provides insights for designing post-hoc ensemble strategies tailored to different tasks.

\begin{table}[t]
    % \vspace{0mm}
    \centering
    % \fontsize{9}{11}\selectfont
    \setlength{\tabcolsep}{1.8mm} % Set column separation to 1mm
    % \caption{Average rank and average total cost across different datasets.}
    
\begin{tabular}{lccc}
\toprule
     & Single-Best & AutoGluon & \sys \\
    \midrule
        CLS & 2.36  & 1.82  & \textbf{1.38}\\
        REG & 2.63  & 2.03  & \textbf{1.33}\\
        ALL & 2.46  & 1.90  & \textbf{1.36} \\

    \bottomrule
\end{tabular}
  
\caption{Average test rank in AutoGluon's search space.}
\label{tab:autogluon_exp}
\end{table}

\subsection{Comparison with Autogluon}
\label{sec:compare_with_autogluon}
AutoGluon represents a state-of-the-art AutoML system with multi-layer stacking. The ``best\_quality" mode is not only AutoGluon's best-performing mode but also surpasses all other AutoML systems in predictive accuracy~\cite{gijsbers2024amlb}. 
To align with its ensemble strategy, AutoGluon employs a more compact CASH search space compared to VolcanoML. 
Therefore, for a fairer comparison, we replicated its search space, which includes 108 zero-shot models with priorities.  
% We first allocate 1 hour for \sys to train and validate models according to the priorities provided by AutoGluon, forming the base model pool, followed by 1 hour for ensemble optimization.
We use AutoGluon to train base models for up to 1 hour, then allocate the remaining time within the 2-hour limit to optimize the ensemble with \sys.
AutoGluon is given a 2-hour budget using ``best\_quality". 
The average test rank in 80 datasets is shown in Table~\ref{tab:autogluon_exp}, where we also include Single-Best as a baseline.
We can conclude that both AutoGluon and \sys outperform Single-Best, and \sys is the best, with an average rank of 1.36.
% Additionally, we calculate the improvement rates of \sys and AutoGluon over Single-Best for each dataset, which are 10.4\% and 4.0\% on average, respectively.
Additionally, we calculate the improvements in test loss achieved by \sys and AutoGluon over Single-Best across all datasets, which are 10.4\% and 4.0\% on average, respectively.
To evaluate the statistical significance of \sys, we conduct a Wilcoxon signed-rank test for the improvements of the two methods, where the value $p = 0.002 \leq 0.05$, indicating that \sys is statistically significantly better than AutoGluon.
To sum up, the consistent success of \sys across candidate pools derived from two systems~(VolcanoML and AutoGluon) highlights its robustness and general applicability.

%% file: tex/relatedworks.tex
\section{Related Work}
\label{sec:related}

\textbf{One-step Ensemble learning.}
% To solve the CASH problem, s
Some methods directly search for an ensemble of models in a single step. Auto-WEKA~\cite{thornton2013auto} adopts a search space covering both ensemble and model hyperparameters, and solves it with Bayesian optimization. Ensemble Optimization~\cite{levesque2016bayesian} maintains a fixed-size ensemble pool and optimizes each model in turn. 
% This strategy suffers from instability due to the risk of adding a bad configuration. 
Autostacker~\cite{chen2018autostacker} and Neural ensemble search~\cite{zaidi2021neural} focus on generating a complex ensemble via evolutionary search.
Another type of work~\cite{shen2022divbo,poduval2024cash} maintains and updates an ensemble selection with consideration of model diversity.
% during the search for the optimal individual model.

%  and TPOT~\cite{olson2016evaluation} 
% focuses on ensembling and cascading to generate complex pipelines, and solves the CASH problem [15] via evolutionary search.

\textbf{Ensemble Selection.} 
% Methods like bagging~\cite{dietterich2000ensemble}, boosting~\cite{freund1996experiments}, and ensemble selection~\cite{caruana2004ensemble} are also used to boost the performance of individual models.
% The two most commonly used and best-performing optimized ensemble strategies are ensemble selection and stacking~\cite{purucker2023assembled}. 
Many AutoML systems, like Auto-Sklearn~\cite{feurer2015efficient}, Auto-Pytorch~\cite{zimmer2021auto}, and MLJAR~\cite{plonska2021mljar} utilize ensemble selection~\cite{okun2009applications}. It iteratively adds models from the pool with replacement to maximize the ensemble performance. CMA-ES~\cite{purucker2023cma} and Q(D)O-ES~\cite{purucker2023q} improve upon ES through choosing models with evolutionary algorithms.
% It starts from an empty ensemble and iteratively adds models from the pool with replacement to maximize the ensemble validation performance.

% Auto-Sklearn~\cite{feurer2015efficient}, AutoGluon~\cite{erickson2020autogluon}, TPOT~\cite{olson2016evaluation}, LightAutoML~\cite{vakhrushev2021lightautoml}, H2O~\cite{ledell2020h2o}, Auto-Pytorch~\cite{zimmer2021auto}, MLJAR~\cite{plonska2021mljar}

%% file: tex/appendix.tex
\section*{Appendix}

\vspace{1em}
\hrule
\vspace{1em}

\appendix
\setcounter{secnumdepth}{2}

\begin{algorithm*}[tb]
\caption{Training a Stacker Model with Dropout and Retain.}
\label{alg:train_with_dropout_and_retain}
\textbf{Input}: The index of the stacker model $i$, Dropout rate $\gamma_0$, Retain or not, Ensemble size $n'$, number of folds $k$, the original training and validation set $\mathcal{D}_{train}$, $\mathcal{D}_{val}$, the predictive features from last layer $\mathcal{Z}_{train}$, $\mathcal{Z}_{val}$, the loss function $\mathcal{L}$.\\
% \textbf{Parameter}: Optional list of parameters\\
\textbf{Output}: Prediction of current stacker model.
\begin{algorithmic}[1] %[1] enables line numbers

\STATE $\text{y\_train} \leftarrow \mathcal{D}_{\text{train}}[:,\ \text{end}]$
\FOR{$j = 1$ {\bf to} $n'$}
   \STATE $\mathcal{L}_j \leftarrow \mathcal{L}(\mathcal{Z}_{train}[:,\ j], \text{y\_train})$
\ENDFOR
\STATE $\mathcal{L}_{\text{min}} \leftarrow \text{Min}(\mathcal{L}_1, \mathcal{L}_2, ..., \mathcal{L}_{n'})$
\STATE $\text{DropoutRates} \leftarrow \text{EmptyList}()$
\FOR{$j = 1$ {\bf to} $n'$}
   \STATE $\text{Append}(\text{DropoutRates}, \frac{\mathcal{L}_{\text{min}}}{\mathcal{L}_j }\gamma_0)$
\ENDFOR

\STATE $\text{folds} \leftarrow \text{SplitIntoKFolds}(\mathcal{D}_{\text{train}}, k)$.
\STATE $\text{training\_prediction} \leftarrow \text{ZeroArray}(\text{length of } \mathcal{D}_{train})$.
\STATE $\text{valid\_prediction} \leftarrow \text{ZeroArray}(\text{length of } \mathcal{D}_{val})$.
\FOR{$j = 1$ {\bf to} $k$}
    \STATE $\text{mask} \leftarrow \text{BernoulliSample}(p=1-\text{DropoutRates})$.
    % \STATE Training Phase
    \STATE $\text{training\_indices} \leftarrow \text{GetAllIndicesExcept}(j, \text{folds})$.
    \STATE $\text{prediction\_indices} \leftarrow \text{GetIndices}(j, \text{folds})$.
    \STATE $\mathcal{D}_{fold\_train} \leftarrow \text{Concatenate}(\mathcal{Z}_{train}[\text{training\_indices},\text{mask}], \mathcal{D}_{train}[\text{training\_indices},:])$. \COMMENT{skip connection}
    \STATE $S_j \leftarrow \text{TrainModel}(\mathcal{D}_{fold\_train})$
    % \STATE $\text{fold\_prediction} \leftarrow \text{Predict}(S_i, \mathcal{D}_{fold\_predict})$
    \STATE $\mathcal{D}_{fold\_predict} \leftarrow \text{Concatenate}(\mathcal{Z}_{train}[\text{prediction\_indices},\text{mask}], \mathcal{D}_{train}[\text{prediction\_indices},:])$.
    \STATE $\text{training\_prediction}[\text{prediction\_indices}] \leftarrow \text{Predict}(S_j, \mathcal{D}_{fold\_predict})$. \COMMENT{predict on the prediction fold of the training set}
    % \STATE Training Phase
    \STATE $\mathcal{D}_{fold\_val} \leftarrow \text{Concatenate}(\mathcal{Z}_{val}[: ,\text{mask}], \mathcal{D}_{val})$.
    \STATE $\text{valid\_prediction} \leftarrow\text{valid\_prediction} + \frac{1}{k}\text{Predict}(S_j, \mathcal{D}_{fold\_val})$. \COMMENT{predict on the validation set}

\ENDFOR

\IF{$\text{Retain or not}$ {\bf is}\ True}
\STATE $\text{y\_valid} \leftarrow \mathcal{D}_{\text{valid}}[:,\ \text{end}]$.
\STATE $\mathcal{L}_{val} \leftarrow \mathcal{L}(\text{valid\_prediction}, \text{y\_val})$.
\STATE $\mathcal{L}_{last} \leftarrow \mathcal{L}(\mathcal{Z}_{val}[:,\ i], \text{y\_val})$.
\IF{$\mathcal{L}_{val} > \mathcal{L}_{last}$}
\STATE $\text{training\_prediction} \leftarrow \mathcal{Z}_{train}[:,\ i]$.
\STATE $\text{valid\_prediction} \leftarrow \mathcal{Z}_{val}[:,\ i]$. \COMMENT{Replace the output with that generated by its predecessor}
\ENDIF 

\ENDIF
    
\STATE \textbf{return} $\text{training\_prediction}$ and $\text{valid\_prediction}$.

\end{algorithmic}
\end{algorithm*}

\input{tex/appendix_theoretical_proofs}

\section{Pseudocode of Training a Stacker Model with Dropout and Retain}
\label{app:pseudcode_of_training}

Algorithm~\ref{alg:train_with_dropout_and_retain} exhibits the pseudocode of training a stacker model with Dropout and Retain in \sys.
First, \sys will calculate the loss between each predictive feature from the last layer and the training labels~(Lines 1-4).
It then determines the dropout rate for each predictive feature based on the training losses and $\gamma_0$.
Next, \sys apply $k$-fold cross-validation to the training set. In each fold, it samples the mask for each predictive feature to retain using a Bernoulli Distribution, and features with a mask of 0 are dropped~(Line 14).
The stacker model is then trained on the training portion of the fold and predicts on the rest part~(Line 15-20). For the validation set,  predictions from all folds are averaged~(Line 21-22).
Upon completing the training for each fold, if the Retain mechanism is enabled, \sys calculates the validation loss of the current stacker model compared to the model in the previous layer at the same position~(Line 24-27). If the current stacker model performs worse, its outputs will be replaced by those generated by its predecessor~(Line 28-30).
In the end, the outputs in both the training set and the validation set will be returned~(Line 33).

\begin{algorithm}[tb]
\caption{Post-hoc Stacking Optimization Procedure}
\label{alg:algorithm}
\textbf{Input}: The search budget $\mathcal{B}$, the candidate base model pool $\mathbb{M}$, the ensemble search space $\mathcal{X}$, the training and validation set $\mathcal{D}_{train}$, $\mathcal{D}_{val}$.\\
% \textbf{Parameter}: Optional list of parameters\\
\textbf{Output}: Best ensemble strategy
\begin{algorithmic}[1] %[1] enables line numbers
\STATE Initialize observations as $D = \emptyset$;.
\STATE Generate predictions of all models in $\mathbb{M}$ on the validation set.
\WHILE{$budget\ \mathcal{B}\ does\ not\ exhaust$}
% \STATE Do some action.
\IF {$|D| < 5$}
\STATE Suggest a random ensemble configuration $\tilde{x} \in \mathcal{X}$.
\ELSE
\STATE Train a surrogate model on current observations $D$.
\STATE Sample configuration candidates from $\mathcal{X}$ and choose the next ensemble configuration by $\tilde{x} \leftarrow \arg\max_{x \in \mathcal{X}}\text{EI}(x)$.
\ENDIF
\STATE Select base model subset $\mathbb{K}$ from $\mathbb{M}$ according to the ensemble size and diversity weight $\tilde{\omega}$ in $\tilde{x}$.
\STATE Build a deep stacking ensemble $f_{\tilde{x}}$ with $\mathbb{K}$.
\STATE Train the ensemble $f_{\tilde{x}}$ on $\mathcal{D}_{train}$ and evaluate its performance on $\mathcal{D}_{val}$ as $\tilde{y}$;
\STATE Update the observations $D \leftarrow D \cup { (\tilde{x}, \tilde{y}) }$;
\ENDWHILE
\STATE \textbf{return} the optimal ensemble configuration ${x}^*$ in $D$.
\end{algorithmic}
\end{algorithm}

\section{Pseudocode and Reuse Strategy of \sys}
\label{app:pseudcode_of_sys}

In this section, we provide the pseudocode of \sys in Algorithm~\ref{alg:algorithm} and the reuse strategy of \sys.

First, \sys will collect the predictions of all models in the candidate pool on the validation set~(Line 1).
Then, while the budget has not been exhausted, 
the Bayesian optimizer fits a \textit{surrogate model} based on observed configurations $D$~(Line 7).
It then suggests the next configuration $\tilde{x}$ to evaluate by maximizing the \textit{acquisition function} $\tilde{x} = \arg\max_{x \in \mathcal{X}}EI(x)$ within the hyperparameter space $\mathcal{X}$~(Line 8). 
According to hyperparameters in $\tilde{x}$, \sys selects the base model subset and builds a stacking ensemble $f_{\tilde{x}}$(Lines 10-11).
After that, \sys trains and evaluates $f_{\tilde{x}}$ to obtain its performance $\tilde{y}$ and augment the observations by $D = D \cup { (\tilde{x}, \tilde{y}) }$.
After exhausting the search budget, the best configuration found in the observations is returned~(Line 15).
Concretely, \sys uses the Probabilistic Random Forest~\cite{hutter2011sequential} as the surrogate of BO and applies the Expected Improvement(EI) as the acquisition function to estimate the performance improvement given an unseen configuration.

By default, \sys saves to disk each stacking ensemble constructed during the optimization process. If a new ensemble configuration $f_{\tilde{x}}$ shares the same base models subset, dropout rate, and retain settings as a previously trained ensemble $f_{x'}$, the cached results can be reused to avoid redundant computation. Specifically, if $f_{\tilde{x}}$ has more stacking layers than $f_{x'}$, training can resume by stacking additional layers on top of $f_{x'}$. Conversely, if the number of layers in $f_{\tilde{x}}$ is fewer than or equal to that in $f_{x'}$, only the blender model corresponding to the target layer needs to be retrained on top of $f_{x'}$’s intermediate output.

\begin{figure}[t!]
	\centering
	% \scalebox{0.95}[.95] {
	\includegraphics[width=0.8\linewidth]{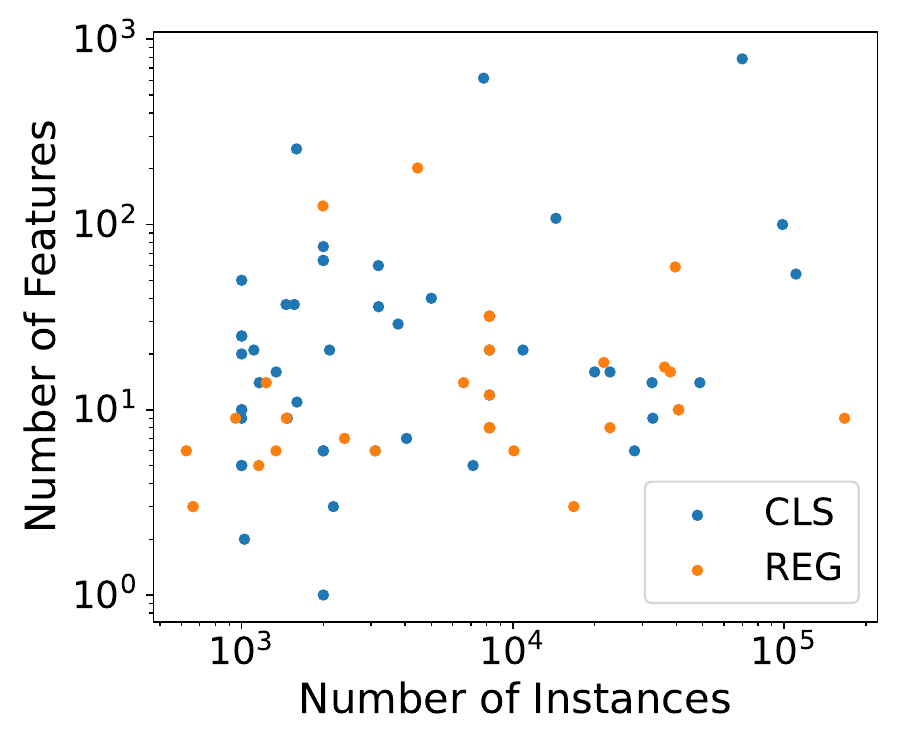}
         % }
	\caption{Datasets by Data Dimensions}
        \label{fig:size_feature}
        % \vspace{-1em}
\end{figure}

\section{Dataset Details}
\label{app:dataset_details}

In Tables~\ref{tab:cls_dataset} and~\ref{tab:reg_dataset}, we provide the details of the datasets used in our experiment.
The datasets are selected according to the following list of criteria: 
\begin{itemize}
    \item \textbf{Representative of real-world}. We limit artificial problems, including only a small selection of such problems, either based on their widespread use (e.g., kr-vs-kp) or because they pose difficult problems. 
    \item \textbf{Diversity in task type}. We include 50 classification datasets, 12 of which are multi-class datasets. We also include 30 regression datasets, which many studies tend to overlook in their experiments~\cite{poduval2024cash,erickson2020autogluon}.
    \item \textbf{Diversity in data Dimensions}. We aim to comprehensively cover datasets ranging from small to large scales in both feature and sample dimensions. As shown in Figure~\ref{fig:size_feature}, the number of features ranges from 1 to 784, and the number of instances ranges from 0.6k to 166k.
    \item \textbf{Diversity in the problem domains}. We do not want the benchmark to skew towards any application domain in particular. Therefore, various domains like software quality problems~(jm1, kc1, pc1, pc3, pc4), handwritten digit recognition~(mfeat-morphological(1), mnist\_784), Aerospace~(delta\_ailerons), Business~(amazon\_employee), computer performance~(puma32H, cpu\_small, cpu\_act, puma8NH), and so forth.
     \item \textbf{Freely available}. All the datasets are collected from OpenML~\cite{vanschoren2014openml}, and we provide the OpenML ID as the identification of the dataset.
\end{itemize}

\section{Search Space of CASH}
\label{app:search_space_of_cash}

The search space for algorithms and feature engineering operators is presented in Tables~\ref{tab:search_space_for_algorithms} and~\ref{tab:search_space_of_FE}, respectively.

\begin{table}[t]
\centering

\begin{subtable}
    \centering
    % \caption{Type of Classifier}
    \begin{tabular}{lcccc}
        \hline
        Type of Classifier & \#$\lambda$ & cat (cond) & cont (cond) \\
        \hline
        AdaBoost & 4 & 1 (-) & 3 (-) \\
        Random forest & 5 & 2 (-) & 3 (-) \\
        Extra trees & 5 & 2 (-) & 3 (-) \\
        Gradient boosting & 7 & 1 (-) & 6 (-) \\
        KNN & 2 & 1 (-) & 1 (-) \\
        LDA & 4 & 1 (-) & 3 (1) \\
        QDA & 1 & - & 1 (-) \\
        Logistic regression & 4 & 2 (-) & 2 (-) \\
        Liblinear SVC & 5 & 2 (2) & 3 (-) \\
        LibSVM SVC & 7 & 2 (2) & 5 (-) \\
        LightGBM & 6 & - & 6 (-) \\
        \hline
    \end{tabular}
\end{subtable}

% \vspace{1em}

\begin{subtable}
    \centering
    % \caption{Type of Regressor}
    \begin{tabular}{lcccc}
        \hline
        Type of Regressor & \#$\lambda$ & cat (cond) & cont (cond) \\
        \hline
        AdaBoost & 4 & 1 (-) & 3 (-) \\
        Random forest & 5 & 2 (-) & 3 (-) \\
        Extra trees & 5 & 2 (-) & 3 (-) \\
        Gradient boosting & 7 & 1 (-) & 6 (-) \\
        KNN & 2 & 1 (-) & 1 (-) \\
        Lasso & 3 & - & 3 (-) \\
        Ridge & 4 & 1 (-) & 3 (-) \\
        Liblinear SVC & 5 & 2 (2) & 3 (-) \\
        LibSVM SVC & 8 & 3 (3) & 5 (-) \\
        LightGBM & 6 & - & 6 (-) \\
        \hline
    \end{tabular}
\end{subtable}
\caption{Search space for ML algorithms. We distinguish categorical (cat) hyperparameters from numerical (cont) ones. The numbers in the brackets are conditional hyperparameters.}
\label{tab:search_space_for_algorithms}
\end{table}

\begin{table}[!t]
\centering
\setlength{\tabcolsep}{1mm} % Set column separation to 1mm
\begin{tabular}{lcccc}
    \hline
    Type of Operator & \#$\lambda$ & cat (cond) & cont (cond) \\
    \hline
    Minmax & 0 & - & - \\
    Normalizer & 0 & - & - \\
    Quantile & 2 & 1 (-) & 1 (-) \\
    Robust & 2 & - & 2 (-) \\
    Standard & 0 & - & - \\
    \hline
    Cross features & 1 & - & 1 (-) \\
    Fast ICA & 4 & 3 (1) & 1 (1) \\
    Feature agglomeration & 4 & 3 (2) & 1 (-) \\
    Kernel PCA & 5 & 1 (1) & 4 (3) \\
    Rand. kitchen sinks & 2 & - & 2 (-) \\
    LDA decomposer & 1 & 1 (-) & - \\
    Nystroem sampler & 5 & 1 (1) & 4 (3) \\
    PCA & 2 & 1 (-) & 1 (-) \\
    Polynomial & 2 & 1 (-) & 1 (-) \\
    Random trees embed. & 5 & 1 (-) & 4 (-) \\
    SVD & 1 & - & 1 (-) \\
    Select percentile & 2 & 1 (-) & 1 (-) \\
    Select generic univariate & 3 & 2 (-) & 1 (-) \\
    Extra trees preprocessing & 5 & 2 (-) & 3 (-) \\
    Linear SVM preprocessing & 5 & 3 (3) & 2 (-) \\
    \hline
\end{tabular}
\caption{Search space of feature engineering operators.}
\label{tab:search_space_of_FE}
\end{table}

\section{Implementation Details}
\label{app:implementation_details}
The Bayesian optimization surrogate is implemented using OpenBox 0.8.1~\cite{li2021openbox,jiang2024openbox}, an open-source toolkit designed for black-box optimization tasks.
We adhered to the open-source versions or the methodologies outlined in the original papers for all other baseline implementations.
For EO, the ensemble size is set to 12.
For Autostacker, the maximum number of layers and the maximum number of nodes per layer are 5 and 3, the population size is set to 200. 
In the case of OptDivBO, the parameter of $\tau$ is 0.2.
For all ensemble selection designs, the ensemble size is fixed at 25, with ensemble weights trained using the base models’ predictions on the validation data.
For all stacking designs, the number of folds for cross-validation during training is set to 5.
All the experiments are conducted with 24 CPU cores and 80G memory.

To fully and fairly utilize all non-test data, ensemble selection-based methods determine ensemble weights on the validation set. After that, the base models are refitted using both the training and validation sets. 
Fixed-strategy stacking methods concatenate the training and validation sets for cross-validation training. Optimized stacking methods (Autostacker, AutoGluon, and \sys) first perform cross-validation on the training set, evaluate performance on the validation set, and finally refit the optimal ensemble using both training and validation sets.

\section{Effectiveness of Deep Stacking Structure}
\label{app:exp_effectiveness_of_struc}

To assess the effectiveness and tuning potential of the deep stacking framework of \sys, this experiment compares the performance of different values for four hyperparameters of deep stacking (using a fixed subset of 30 base models selected with a diversity weight of 0.3). 
For the dropout rate, values of 0, 0.1, 0.2, 0.3, and 0.4 are tested; 
For the number of layers, values of 1, 2, 3, 4, and 5 are considered. 
For the blender model, ensemble selection~(ES), Linear, and LightGBM~(LGB) are tested.
And ``retain or not" takes True or False
Due to the vast configuration space, we conducted separate experiments for each hyperparameter instead of a full grid search. 
Specifically, for each hyperparameter, we held the other hyperparameters at their median values from the set of possible values, then iterated over all values for the target hyperparameter to construct stacking ensembles. The optimal value (\textsc{Opt}) was selected based on validation set performance. 
This optimal configuration was then compared to all fixed-value ensembles for the hyperparameter on the test set, with the average rank calculated across 80 datasets. Notably, when fixing the Blender model, we used Linear, and when fixing the Retain strategy, we set it to True.

Figure~\ref{fig:struc} shows the average test rank across each hyperparameter’s values, where each hyperparameter experiment was conducted independently and compared separately.
From this analysis, we observe that 1) for the three hyperparameters—dropout rate, number of layers, and blender model—no specific value consistently outperforms others. This is evident from the fact that the average rank differences between the top two fixed value schemes are only 0.03, 0.10, and 0.18, excluding \text{Opt}. Therefore, the optimal values for these hyperparameters are highly task-dependent. Correspondingly, tuning these parameters significantly improves performance, as \textsc{Opt} surpasses the best fixed scheme by 0.83, 0.73, and 0.61 in average rank, far exceeding the differences between fixed values. 2) For the retain or not hyperparameter, "True" already shows a clear advantage over "False", with average ranks of 1.5 and 1.85, respectively. However, selecting values based on the validation set can further improve the rank to 1.3.
In summary, all four hyperparameters of the deep stacking framework merit tuning to identify task-specific optimal values.

\begin{figure}[t!]
	\centering
	% \scalebox{0.95}[.95] {
	\includegraphics[width=0.7\linewidth]{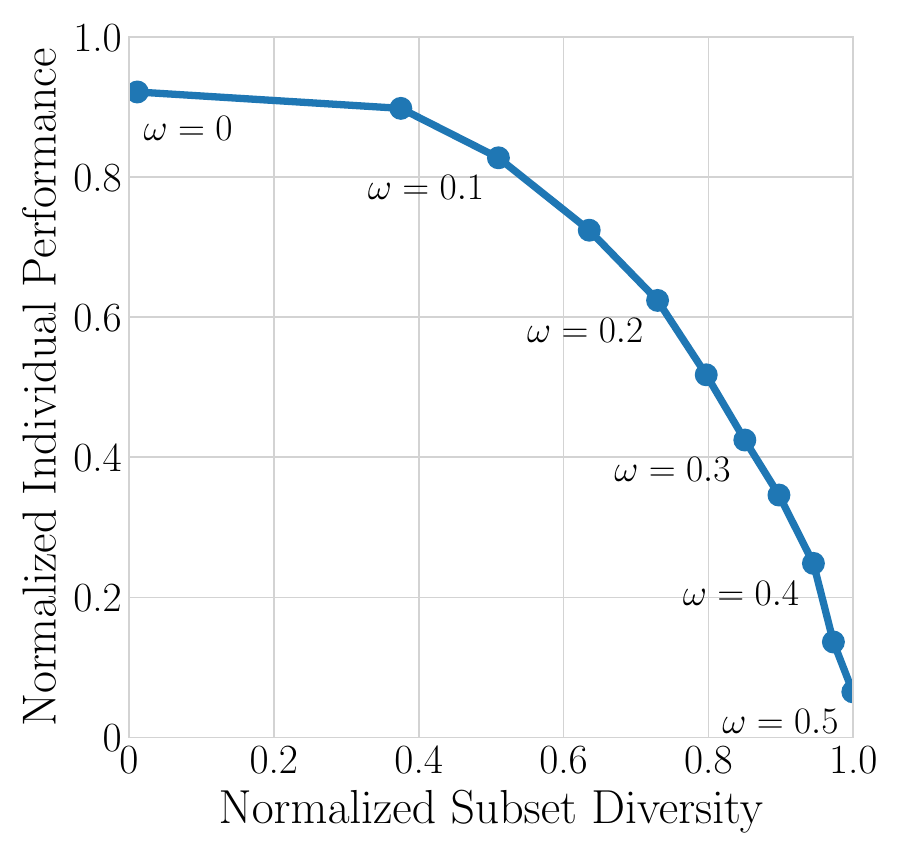}
         % }
	\caption{Relationship between normalized average individual performance and normalized diversity among the base model subset across 80 datasets for varying diversity weights. We execute base model subset selection with 11 diversity weights ranging from 0 to 0.5 in increments of 0.05. For each dataset, the average performance of 30 selected base models with different diversity weights was calculated and normalized using min-max scaling from 0 to 1. The subset diversity was assessed and normalized in the same way. 
}
        \label{fig:perf_with_div}
        % \vspace{-1em}
\end{figure}

\section{Trade-off Between Individual Model Performance and Inter-model Diversity}
\label{app:exp_trade-off}

In this experiment, we aim to validate the base model selection algorithm of \sys by exploring the impact of varying diversity weights. By setting the ensemble size to 30 and adjusting the diversity weights, ranging from 0 to 0.5 in increments of 0.05, we investigate the trade-off between individual model performance and subset diversity. 

Specifically, for each diversity weight value, we calculated the average individual performance of the selected base models, defined as the average of their metrics on the task. For classification tasks, this metric is $accuracy$, and for regression tasks, it is the negative of $mean\ squared\ error$~(MSE). A higher value indicates better individual performance.
Similarly, we assessed the diversity within the selected base models subset by computing the average pair-wise diversity, derived from the Equation~\ref{eq:div}. We use its negative to describe diversity. A higher value indicates greater diversity.

To synthesize results across all tasks, we apply min-max normalization to both individual performance and subset diversity for each dataset across the 11 diversity weight values, scaling them to a range of 0 to 1. Subsequently, we average these normalized results over 80 tasks and show the results in Figure~\ref{fig:perf_with_div}.
We observed that: 1) The 11 points, from left to right, represent the increasing diversity weights from the smallest to the largest value. Correspondingly, the individual performance of the selected base model subset decreases, while diversity increases, demonstrating a clear trade-off.
2) The trade-off curve is convex towards the upper right. This indicates that as the diversity weight increases, diversity gains diminish compared to the loss in individual performance. Initially, small increases in dropout improve diversity without significantly sacrificing performance. However, further increases lead to greater performance losses with smaller diversity gains.

\begin{figure}[tb]
% % \vspace{-3mm}
\centering

\label{fig:maxinum_weight}
\includegraphics[width=0.6\linewidth]{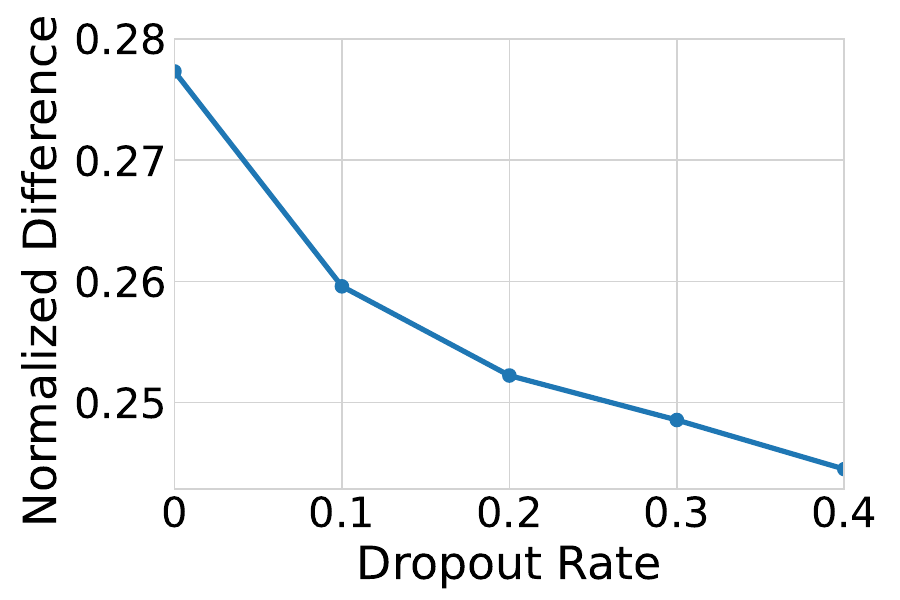}

\caption{Normalized difference between training and test performance changing with dropout rate.}
\label{fig:train_test_diff}
% \vspace{-1em}
\end{figure}

\begin{figure}[tb]
% % \vspace{-3mm}
\centering

\label{fig:maxinum_weight}
\includegraphics[width=0.6\linewidth]{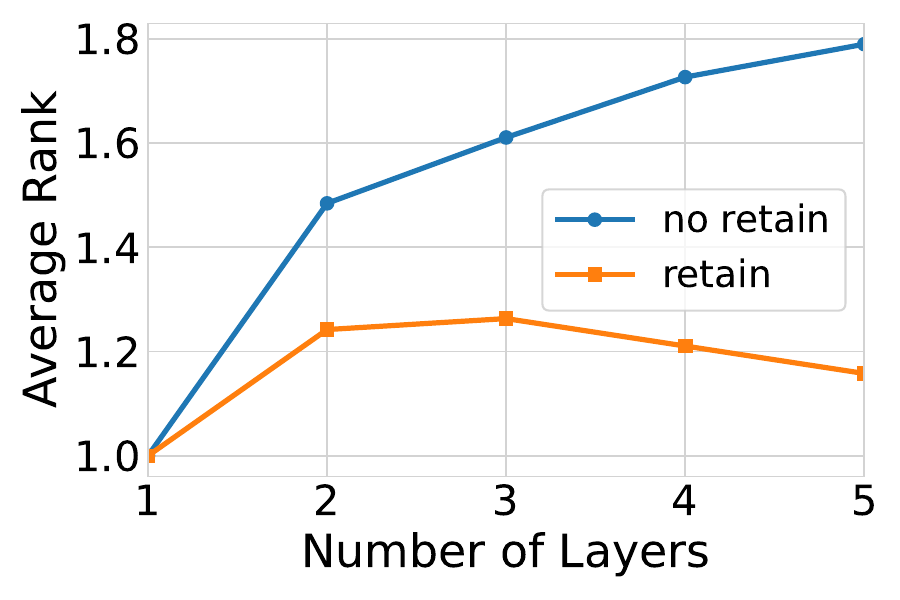}

\caption{Average rank of output with(out) Retain.}
\label{fig:avg_rank_of_output}
% \vspace{-1em}
\end{figure}

\begin{figure*}[t!]
\centering  
\subfigure[Optimal base model selection strategy across all tasks.]{
\label{fig:base_choose_with_tasktype}
\includegraphics[width=0.3\linewidth]{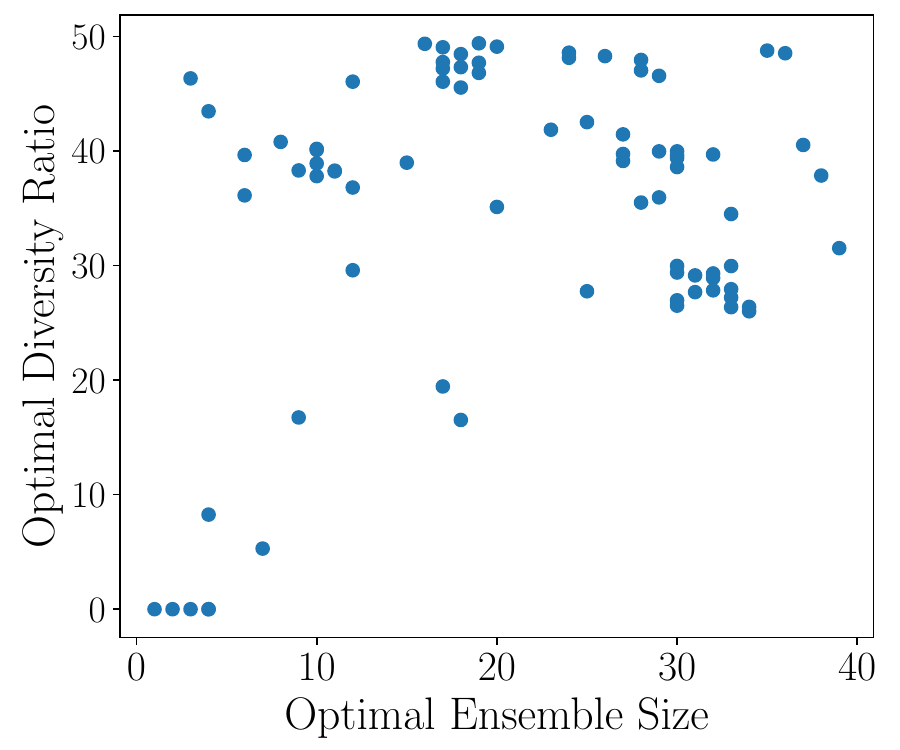}
}
% \hspace{0.02\linewidth}
\subfigure[Optimal number of layers for different task types.]{
\label{fig:layer_with_tasktype}
\includegraphics[width=0.3\linewidth]{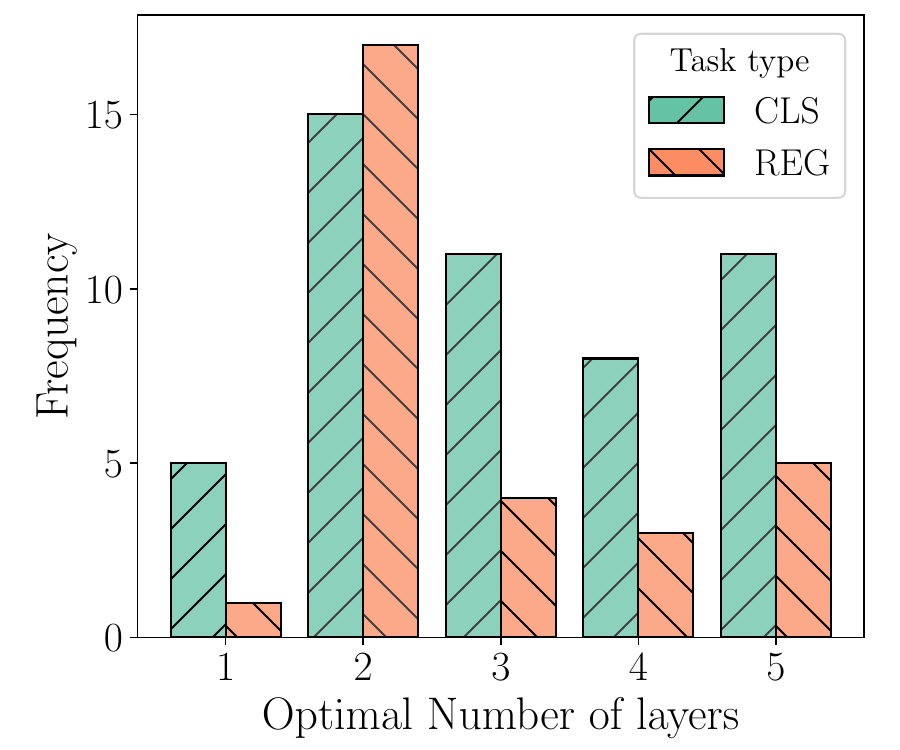}
}
% \hspace{0.02\linewidth}
\subfigure[Optimal dropout ratios for tasks with different dataset sizes.]{
\label{fig:dropout_with_instance}
\includegraphics[width=0.3\linewidth]{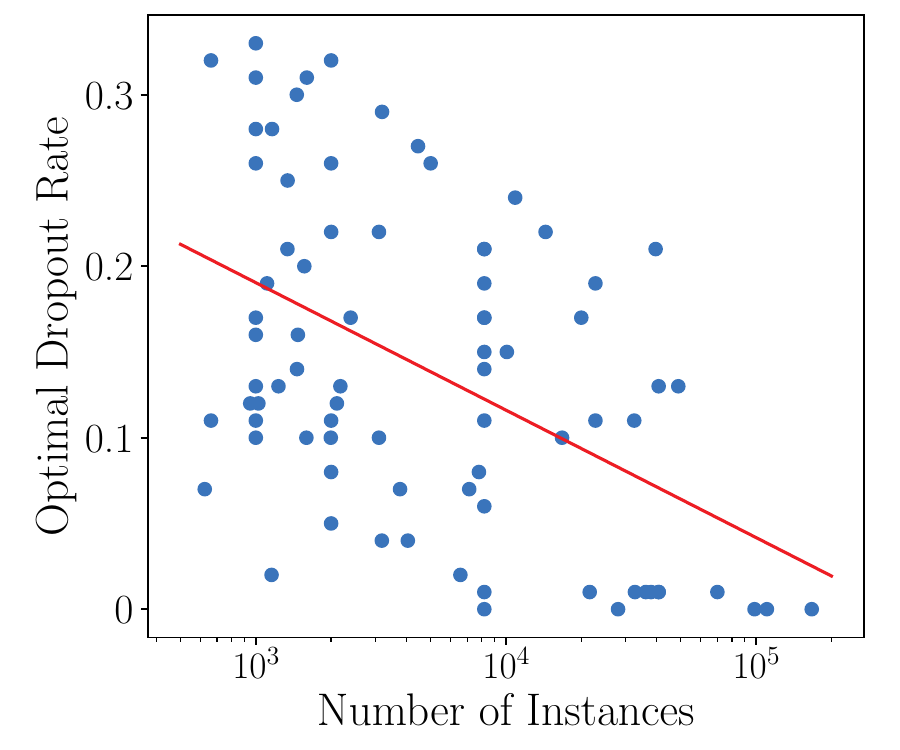}
}
% \hspace{0.03\linewidth}
\subfigure[Optimal blender model for tasks with different dataset sizes.]{
\label{fig:blender_model_with_instance}
\includegraphics[width=0.3\linewidth]{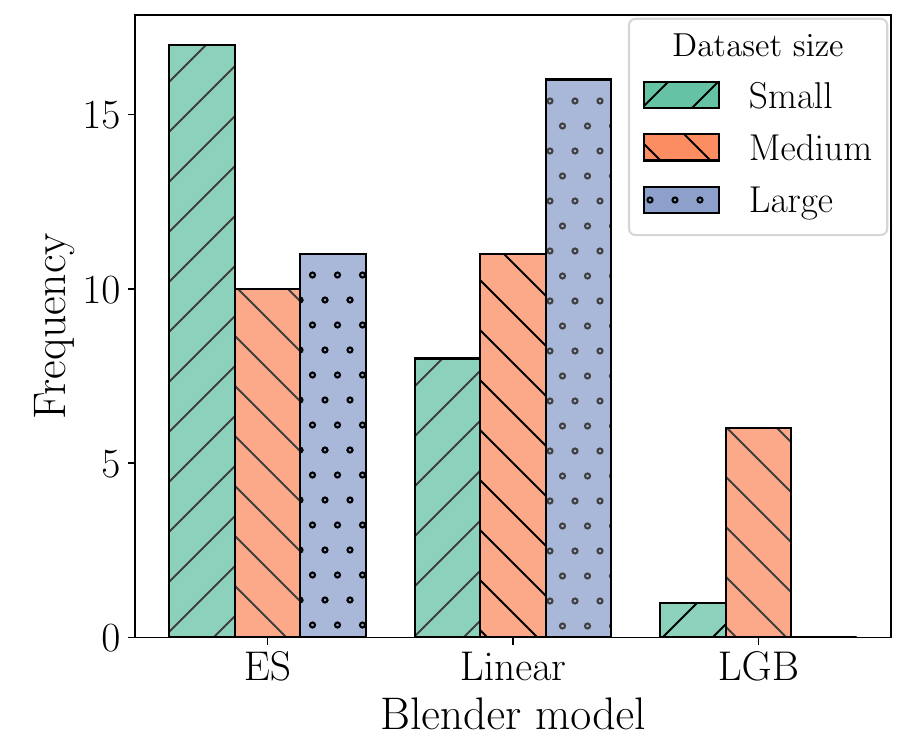}
}
\hspace{0.03\linewidth}
\subfigure[Optimal number of layers with optimal retain or not parameter.]{
\label{fig:layer_with_retain}
\includegraphics[width=0.3\linewidth]{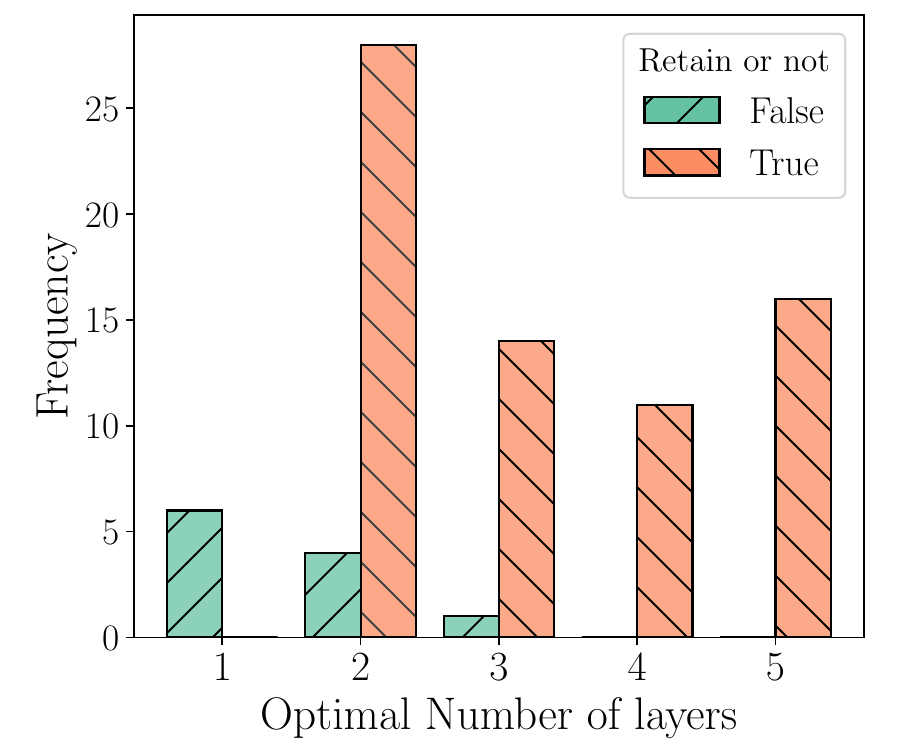}
}

\caption{Optimal Configuration Analysis.}

\label{fig:case_study}
% \vspace{1mm}
\end{figure*}

\section{Effectiveness of Dropout Mechanism}
\label{app:exp_effectiveness_of_dropout}
To highlight the effect of Dropout on the overfitting issue, Figure~\ref{fig:train_test_diff} shows the average normalized difference between training and test performance with dropout rate changing from 0 to 0.4 across 80 datasets.
Specifically, the normalized difference is the ratio of the absolute value of the difference between the training loss and test loss to the absolute value of the training loss: $|\mathcal{L}_{train} - \mathcal{L}_{test}| / |\mathcal{L}_{train}|$.
We observe that as the dropout rate increases, the gap between the training and test performance gradually narrows, indicating a reduction in overfitting.

\section{Effectiveness of Retain Mechanism}
\label{app:exp_effectiveness_of_retain}

To highlight the effect of predictive feature quality on ensemble performance, Figure~\ref{fig:avg_rank_of_output} uses a linear model as the blender and shows the average test rank of stacking with and without the Retain mechanism, as stacking depth increases. The results indicate that the Retain mechanism performs increasingly better, widening the gap over the baseline.

\begin{figure*}[tb]
% % \vspace{-3mm}
\centering

\subfigure[Proportion of max weight.]{
\includegraphics[width=0.47\linewidth]{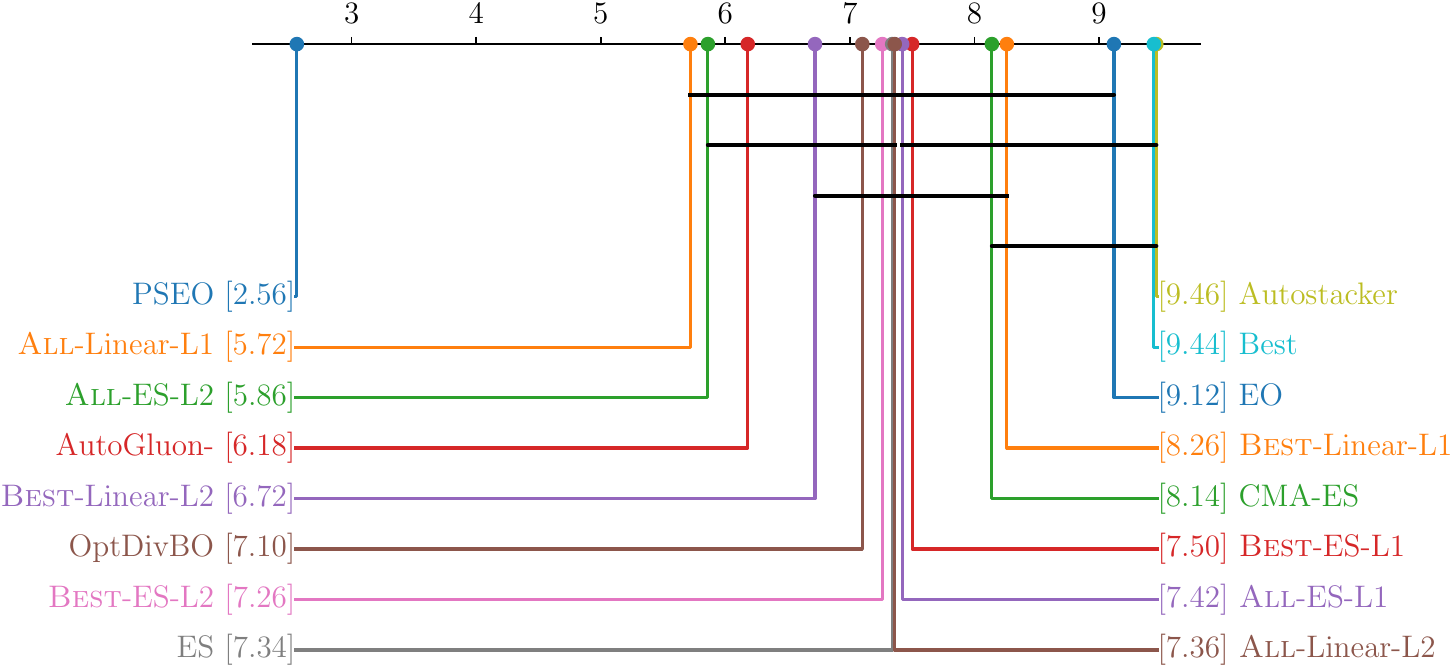}
\label{fig:cd}
}
\subfigure[Improvement of stackers.]{
\includegraphics[width=0.475\linewidth]{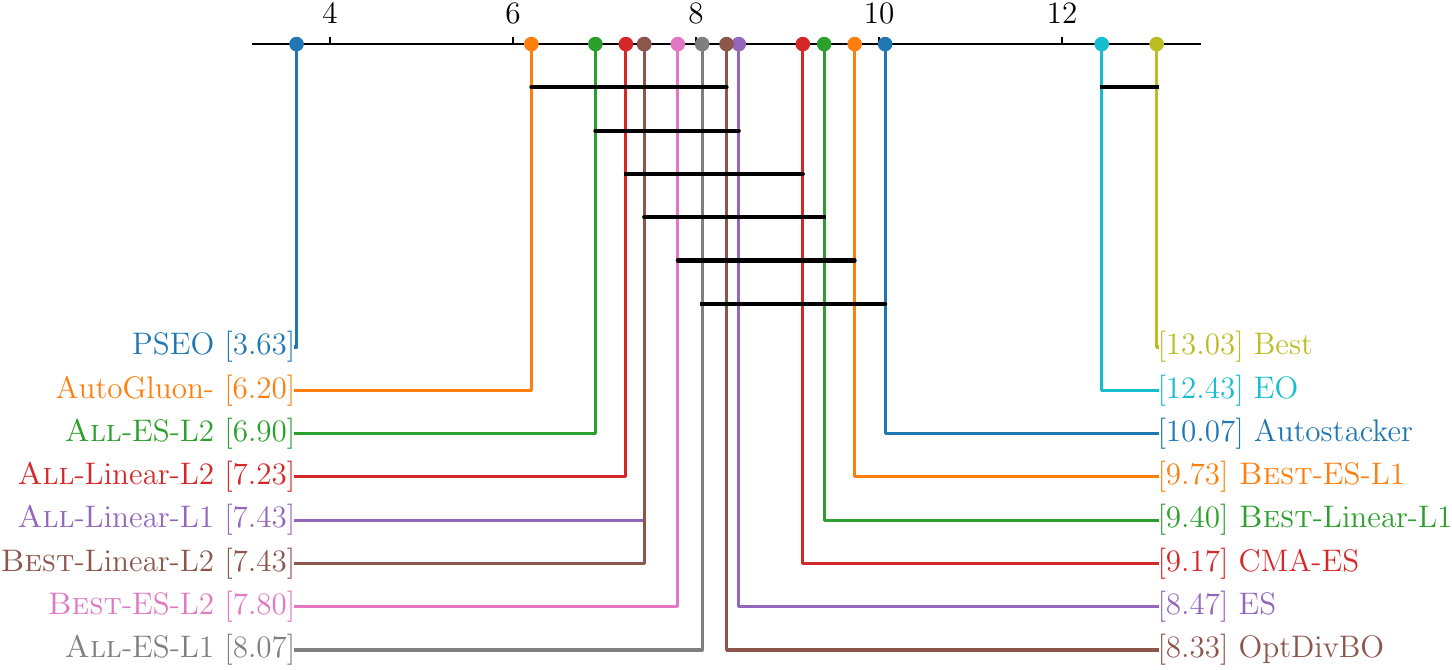}
\label{fig:cd_rgs}
}
\caption{Effectiveness of Dropout and Retain mechanisms.}
\label{}
% \vspace{-1em}
\end{figure*}

\section{Critical Difference Diagram}
Figure~\ref{fig:cd} shows the results of the Conover
post hoc test using critical difference plots, whereby a horizontal bar connects methods that are not significantly different.

\begin{table}[t]
    % \vspace{0mm}
    \centering
    % \fontsize{9}{\baselineskip}\selectfont
    \setlength{\tabcolsep}{1.2mm} % Set column separation to 1mm
    % \caption{Average rank and average total cost across different datasets.}
    
\begin{tabular}{lcccc}
\toprule
     & Single-Best & w/o Retain & w/o Dropout & \sys \\
    \midrule
        CLS & 3.24 & 2.04 & 1.84 & \textbf{1.46} \\
        REG & 3.67 & 2.30 & 2.27 & \textbf{1.70} \\
        ALL & 3.40 & 2.14 & 2.00 & \textbf{1.55} \\

    \bottomrule
\end{tabular}
  
\caption{Average test rank.}
\label{tab:aba_exp}
\end{table}

\section{Ablation Study}
\label{app:exp_ablation}
In this section, we present the effects of two mechanisms in \sys via an ablation study in Table~\ref{tab:aba_exp}.
We evaluate: 1) Single-Best; 2) w/o Retain: \sys without the Retain mechanism; 3) w/o Dropout: \sys without the Dropout mechanism; 4) \sys.
We further calculate the improvement rates of each method over Single-Best for each dataset.
The average improvements of w/o Retain, w/o Dropout, and \sys are 9.8\%, 7.7\%, and 11.0\%, respectively.
First, removing either mechanism leads to lower average rankings compared to the full \sys, and also results in reduced improvement over the Single-Best baseline.
Secondly, the ranking of w/o Retain is lower than that of w/o Dropout, but its improvement rate over Single-Best is higher. This suggests that the Dropout can elevate the upper bound of stacking performance on datasets prone to overfitting, while the Retain mechanism maintains the lower bound for most datasets.
To summarize, the two mechanisms of \sys show performance improvement. 
It indicates that the deep stacking framework of \sys is a reasonable design that provides a flexible framework adaptable to various tasks, along with optional strategies for addressing specific issues.

\section{Optimal Configuration Analysis}
\label{app:exp_optimal_conf_ana}

In this section, we analyze the optimal ensemble strategies identified by \sys across 80 ML tasks, aiming to provide insights for designing post-hoc ensemble strategies tailored to different tasks. Figure~\ref{fig:case_study} presents the results of five analyses, each of which will be detailed individually below.

1) We extracted the base model selection strategies from the optimal ensembles across 80 tasks to examine the relationship between optimal ensemble size and optimal diversity weight. The scatter plot is shown in Figure~\ref{fig:base_choose_with_tasktype}. 
We observe that \textbf{most optimal configurations have a high diversity weight} (greater than 0.25), indicating that diversity can enhance ensemble performance. Additionally, the absence of points in the lower right corner suggests that \textbf{large ensemble sizes coupled with low diversity weights are typically not optimal}, as this situation may lead to overfitting.

2) For the number of stacking layers, we plotted the distribution of optimal ensemble stacking layers separately for classification and regression tasks, as shown in Figure~\ref{fig:layer_with_tasktype}.
We observe that \textbf{the most common optimal ensemble stacking layer number is two}, particularly for regression tasks, with 17 out of 30 tasks having an optimal stacking of two layers, while the frequency for other layer counts does not exceed 5. In contrast, for classification tasks, the frequency of optimal stacking layers ranging from three to five is relatively closer to that of two. \textbf{This suggests that classification tasks have greater potential for deeper stacking, whereas regression tasks may experience instability and suboptimal performance with higher stacking layers}.

3) Regarding the choice of dropout rate, we plotted the relationship between the optimal dropout rate and the number of instances in the dataset for 80 tasks, as shown in Figure~\ref{fig:dropout_with_instance}. The red line represents the fitted linear trend. We can observe that, generally, as the dataset size increases, the optimal dropout rate tends to decrease. \textbf{This indicates that smaller datasets require the Dropout mechanism more}, as overfitting is more likely to occur in such tasks.

4) For the blender model, we sorted the 80 datasets by the number of instances and divided them into three categories: small datasets (less than 1.6k), medium datasets (2k to 8k), and large datasets (greater than 10k). We then plotted the distribution of optimal Blender models for each category, as shown in Figure~\ref{fig:blender_model_with_instance}.
We observe that \textbf{the LGB model is generally not the optimal choice}, with only 7 out of 80 tasks selecting it as the best blender model. \textbf{ES and Linear perform similarly overall, with ES showing a clear advantage on smaller datasets. As the dataset size increases, the performance of the Linear tends to improve}.

5) Regarding the Retain mechanism, we examine the relationship between whether Retain is used in the optimal configuration and the number of stacking layers. Figure~\ref{fig:layer_with_retain} shows the distribution of the optimal number of stacking layers where Retain is set to True and False, respectively.
We can conclude that the \textbf{Retain mechanism is recommended to be enabled}, as it is set to True in the optimal configuration for 69 out of 80 tasks. Additionally, \textbf{to achieve better ensemble performance with higher stacking layers, Retain is essential}, because when the optimal number of stacking layers exceeds one, Retain is almost always True.

\section{Test Performance Per Dataset Per Method}
\label{app:exp_performance_perdataset_permethod}
This section provides the exact performance of each method per dataset. 
Tables~\ref{tab:cls_results_details_part1}, \ref{tab:cls_results_details_part2}, and \ref{tab:cls_results_details_part3} present the test errors for all methods across 50 classification tasks, where test error equals $1-accuracy$. 
Tables~\ref{tab:rgs_results_details_part1}, \ref{tab:rgs_results_details_part2}, and \ref{tab:rgs_results_details_part3} show the test errors for all methods across 30 regression tasks, where test error equals the $mean\ squared\ error$~(MSE).

\section{Detailed Information about Reproducibility Checklist}
\label{app:detailed_checklist}
% In this section, we provided detailed information about each reproducibility criterion in the reproducibility checklist.

\begin{itemize}
\item 1. General Paper Structure
\item 1.1. Includes a conceptual outline and/or pseudocode description of AI methods introduced. [yes]. See Algorithm~\ref{alg:train_with_dropout_and_retain} in Appendix~\ref{app:pseudcode_of_training} and Algorithm~\ref{alg:algorithm} in Appendix~\ref{app:pseudcode_of_sys}.
\item 1.2. Clearly delineates statements that are opinions, hypothesis, and speculation from objective facts and results. [yes].
\item 1.3. Provides well-marked pedagogical references for less-familiar readers to gain background necessary to replicate the paper. [yes]. We have provided references to foundational and cutting-edge articles in the field, along with sources for all experimental data.
\item 2. Theoretical Contributions
\item 2.1. Does this paper make theoretical contributions? [yes]. See Theorem~\ref{the:decrease} and Theorem~\ref{the:expectation}.
\item 2.2. All assumptions and restrictions are stated clearly and formally. [yes]. See  Theorem~\ref{the:decrease}, Theorem~\ref{the:expectation} and Section~\ref{sec:deep_stacking_ensemble}.
\item 2.3. All novel claims are stated formally (e.g., in theorem statements). [yes]. See  Theorem~\ref{the:decrease} and Theorem~\ref{the:expectation}.
\item 2.4. Proofs of all novel claims are included. [yes]. See the proof in Appendix~\ref{app:proof}.
\item 2.5. Proof sketches or intuitions are given for complex and/or novel results. [yes]. See the proof sketch in Appendix~\ref{app:proof}.
\item 2.6. Appropriate citations to theoretical tools used are given. [yes].
\item 2.7. All theoretical claims are demonstrated empirically to hold. [yes].
\item 2.8. All experimental code used to eliminate or disprove claims is included. [yes].See the code in ``PSEO\_code/examples/benchmark\_ana\_dropout.py''.
\item 3. Dataset Usage
\item 3.1. Does this paper rely on one or more datasets? [yes]. We rely on 80 public datasets in our experiment in Section~\ref{sec:experiment}.
\item 3.2. A motivation is given for why the experiments are conducted on the selected datasets. [yes]. We provide the list of criteria for selecting datasets in Appendix~\ref{app:dataset_details}.
\item 3.3. All novel datasets introduced in this paper are included in a data appendix. [NA]. We haven't introduced novel datasets.
\item 3.4. All novel datasets introduced in this paper will be made publicly available upon publication of the paper with a license that allows free usage for research purposes. [NA]. We haven't introduced novel datasets.
\item 3.5. All datasets drawn from the existing literature (potentially including authors’ own previously published work) are accompanied by appropriate citations. [yes]. All the datasets are collected from OpenML~\cite{vanschoren2014openml}, and we provide the OpenML ID as the identification of the dataset.
\item 3.6. All datasets that are not publicly available are described in detail, with explanation why publicly available alternatives are not scientifically satisficing. [NA].All the datasets are publicly available in OpenML.
\item 4. Computational Experiments
\item 4.1. Does this paper include computational experiments. [yes]. See Section~\ref{sec:experiment}.
\item 4.2. This paper states the number and range of values tried per (hyper-) parameter during development of the paper, along with the criterion used for selecting the final parameter setting. [yes]. We decide the hyperparameters through tuning.
\item 4.3. Any code required for pre-processing data is included in the appendix. [yes]. See the code in ``PSEO\_code/examples/benchmark\_pseo.py'' file.
\item 4.4. All source code required for conducting and analyzing the experiments is included in a code appendix. [yes]. See the code in ``PSEO\_code/examples/benchmark\_pseo.py'' and ``PSEO\_code/examples/benchmark\_ana\_rank.py''.
\item 4.5. All source code required for conducting and analyzing the experiments will be made publicly available upon publication of the paper with a license
that allows free usage for research purposes. [yes]. We will open-source the code in GitHub upon publication.
\item 4.6. All source code implementing new methods have
comments detailing the implementation, with references to the paper where each step comes from. [yes]. See the code in ``PSEO\_code/pseo'' directory.
\item 4.7. If an algorithm depends on randomness, then the method used for setting seeds is described. [NA].
\item 4.8. This paper specifies the computing infrastructure
used for running experiments (hardware and software), including GPU/CPU models; amount of
memory; operating system; names and versions of
relevant software libraries and frameworks. [yes]. See Appendix~\ref{app:implementation_details}.
\item 4.9. This paper formally describes evaluation metrics used and explains the motivation for choosing these metrics. See Section~\ref{sec:experiment_setup}.
\item 4.10. This paper states the number of algorithm runs used to compute each reported result. [yes].
\item 4.11. Analysis of experiments goes beyond single dimensional summaries of performance (e.g., average; median) to include measures of variation, confidence, or other distributional information. [yes]. See Figure~\ref{fig:main_exp} and Section~\ref{sec:main_exp}.
\item 4.12. The significance of any improvement or decrease in performance is judged using appropriate statistical tests (e.g., Wilcoxon signed-rank). [yes]. See Section~\ref{sec:compare_with_autogluon}.
\item 4.13. This paper lists all final (hyper-)parameters used for each model/algorithm in the paper’s experiments. [yes]. See Appendix~\ref{app:implementation_details}.

\end{itemize}

\begin{table*}[t]
\centering
% \vspace{-1mm}
\begin{tabular}{lccccc}
\toprule
Datasets & Task Type & OpenML ID & Classes & Samples & Features \\
\midrule
fri\_c0\_1000\_5 & CLS & 609 & 2 & 1000 & 5\\
fri\_c2\_1000\_5 & CLS & 599 & 2 & 1000 & 5\\
fri\_c4\_1000\_10 & CLS & 623 & 2 & 1000 & 10\\
fri\_c1\_1000\_50 & CLS & 583 & 2 & 1000 & 50\\
credit & CLS & 46116 & 2 & 1000 & 9\\
fri\_c3\_1000\_10 & CLS & 608 & 2 & 1000 & 10\\
fri\_c0\_1000\_25 & CLS & 598 & 2 & 1000 & 25\\
credit-g & CLS & 31 & 2 & 1000 & 20\\
fri\_c3\_1000\_25 & CLS & 586 & 2 & 1000 & 25\\
rmftsa\_sleepdata(1) & CLS & 679 & 4 & 1024 & 2\\
pc1 & CLS & 1068 & 2 & 1109 & 21\\
colleges\_aaup & CLS & 488 & 2 & 1161 & 14\\
analcatdata\_halloffame & CLS & 454 & 2 & 1340 & 16\\
pc4 & CLS & 1049 & 2 & 1458 & 37\\
cmc & CLS & 23 & 3 & 1473 & 9\\
pc3 & CLS & 1050 & 2 & 1563 & 37\\
semeion & CLS & 1501 & 10 & 1593 & 256\\
winequality\_red & CLS & 42184 & 6 & 1599 & 11\\
mfeat-morphological(1) & CLS & 18 & 10 & 2000 & 6\\
mfeat-morphological(2) & CLS & 962 & 2 & 2000 & 6\\
mfeat-karhunen(2) & CLS & 1020 & 2 & 2000 & 64\\
mfeat-karhunen(1) & CLS & 16 & 10 & 2000 & 64\\
mfeat-fourier(1) & CLS & 14 & 10 & 2000 & 76\\
balloon & CLS & 512 & 2 & 2001 & 1\\
kc1 & CLS & 1067 & 2 & 2109 & 21\\
quake & CLS & 209 & 2 & 2178 & 3\\
space\_ga & CLS & 507 & 2 & 3107 & 6\\
splice & CLS & 46 & 3 & 3190 & 60\\
kr-vs-kp & CLS & 3 & 2 & 3196 & 36\\
hypothyroid(2) & CLS & 1000 & 2 & 3772 & 29\\
analcatdata\_supreme & CLS & 504 & 2 & 4052 & 7\\
waveform-5000(2) & CLS & 979 & 2 & 5000 & 40\\
delta\_ailerons & CLS & 803 & 2 & 7129 & 5\\
isolet & CLS & 300 & 26 & 7797 & 617\\
puma32H & CLS & 308 & 2 & 8192 & 32\\
cpu\_small & CLS & 227 & 2 & 8192 & 12\\
puma8NH & CLS & 225 & 2 & 8192 & 8\\
cpu\_act & CLS & 197 & 2 & 8192 & 21\\
jm1 & CLS & 1053 & 2 & 10885 & 21\\
sylva\_prior & CLS & 1040 & 2 & 14395 & 108\\
letter(2) & CLS & 977 & 2 & 20000 & 16\\
house\_16H & CLS & 574 & 2 & 22784 & 16\\
kropt & CLS & 184 & 18 & 28056 & 6\\
adult-census & CLS & 1119 & 2 & 32561 & 14\\
amazon\_employee & CLS & 4135 & 2 & 32769 & 9\\
mv & CLS & 344 & 2 & 40768 & 10\\
adult & CLS & 179 & 2 & 48842 & 14\\
mnist\_784 & CLS & 554 & 10 & 70000 & 784\\
vehicle\_sensIT & CLS & 357 & 2 & 98528 & 100\\
covertype & CLS & 180 & 7 & 110393 & 54\\
\bottomrule
\end{tabular}
\caption{Basic information of classification dataset.}
\label{tab:cls_dataset}
\end{table*}

\begin{table*}[t]
\centering
% \vspace{-1mm}
\begin{tabular}{lccccc}
\toprule
Datasets & Task Type & OpenML ID & Samples & Features \\
\midrule
strikes & REG & 549 & 625 & 6\\
disclosure\_x\_noise & REG & 704 & 662 & 3\\
disclosure\_z & REG & 699 & 662 & 3\\
stock & REG & 223 & 950 & 9\\
socmob & REG & 541 & 1156 & 5\\
Moneyball & REG & 41021 & 1232 & 14\\
insurance & REG & 43463 & 1338 & 6\\
weather\_izmir & REG & 42369 & 1461 & 9\\
us\_crime & REG & 42730 & 1994 & 126\\
debutanizer & REG & 23516 & 2394 & 7\\
space\_ga & REG & 507 & 3107 & 6\\
mtp & REG & 405 & 4450 & 202\\
wind & REG & 503 & 6574 & 14\\
bank32nh & REG & 558 & 8192 & 32\\
bank8FM & REG & 572 & 8192 & 8\\
cpu\_act & REG & 197 & 8192 & 21\\
cpu\_small & REG & 227 & 8192 & 12\\
kin8nm & REG & 189 & 8192 & 8\\
puma32H & REG & 308 & 8192 & 32\\
puma8NH & REG & 225 & 8192 & 8\\
sulfur & REG & 23515 & 10081 & 6\\
rainfall\_bangladesh & REG & 41539 & 16755 & 3\\
kc\_house\_data & REG & 42079 & 21613 & 18\\
house\_8L & REG & 218 & 22784 & 8\\
NewFuelCar & REG & 41506 & 36203 & 17\\
electricity\_prices\_ICON & REG & 1168 & 38014 & 16\\
OnlineNewsPopularity & REG & 42724 & 39644 & 59\\
2dplanes & REG & 215 & 40768 & 10\\
mv & REG & 344 & 40768 & 10\\
black\_friday & REG & 41540 & 166821 & 9\\
\bottomrule
\end{tabular}
\caption{Basic information of regression dataset.}
\label{tab:reg_dataset}
\end{table*}

\input{tex/perf_table}

%% file: tex/appendix_theoretical_proofs.tex
\section{Theoretical Proofs}
\label{app:proof}

In this section, we provide the proof of Theorem~\ref{the:decrease}. 

\textbf{Proof sketches}: To prove Theorem~\ref{the:decrease}, we 1) first establish Lemma~\ref{lem:no_changing_beta}, showing that the expected value of the sampled weight is the product of the probability of retention and the original weight without Dropout, leveraging the orthogonality of predictions and linearity of expectation. 2) Next, for Theorem~\ref{the:expectation}, we apply the law of large numbers to demonstrate that the average estimator converges almost surely to the product as the times of sampling $\to \infty$, using the result from the lemma.
3) Finally, with Lemma~\ref{lem:no_changing_beta} and Theorem~\ref{the:expectation}, we can represent the proportion of $|\tilde{\beta}_1|$ in the sum of all others, and prove it's monotonically non-decreasing with respect to $\gamma_0$.

\begin{lemma}
    \label{lem:no_changing_beta}
    Consider a set of uncorrelated base models predictions $\mathcal{Z} = \{z_1, z_2, ..., z_{n'} \}$
    in an ensemble with prediction $z_{\text{ens}} = \sum_i \beta_iz_i$.
    If all the $z_i$ are orthogonal, we perform random samplings of $\mathcal{Z}$ (each prediction retained with probability $p_i$).
    If the prediction is dropped, its weight is filled with 0.
    Then the augmented estimator $\hat{\beta}_i$ obtained from random sampling will satisfy:
    $$
        \mathbb{E}[\hat{\beta}_i] = p_i\hat{\beta}_{\text{old}, i}
    $$
    where $\hat{\beta}_{\text{old}}$ is the original weight without sampling.
\end{lemma}

\begin{proof}
    Assume $X$ is a matrix where the $i$-th column represents $z_i$, and $Y$ is the ground truth. Then the ensemble weights $\hat{\beta}_{\text{old}}$ can be computed with closed form solution of the objective $\text{min}_{\beta} ||Y-X \beta ||_2$ as:
    $$
        \hat{\beta}_{\text{old}} = (X^T X)^{-1} X^T Y.
    $$
    
    Since $X^T X$ is diagonal, its inverse is also diagonal:
    $$
        (X^T X)^{-1} = \begin{pmatrix}
            \frac{1}{g_1} & & \\
            & \ddots & \\
            & & \frac{1}{g_d}
        \end{pmatrix}.
    $$
    where $g_i = \sum_j x_{ji}^2$ is the sum of squares of the $i$-th base prediction.

    % And $X^T Y = \begin{pmatrix} \sum_i x_{i1} y_i \\ \vdots \\ \sum_i x_{id} y_i \end{pmatrix}$, 
    thus
    $$
        \hat{\beta}_{\text{old}} = 
        \begin{pmatrix}
            \frac{1}{g_1} & & \\
            & \ddots & \\
            & & \frac{1}{g_d}
        \end{pmatrix}
        \begin{pmatrix}
            \sum_i x_{i1} y_i \\
            \vdots \\
            \sum_i x_{id} y_i
        \end{pmatrix}
        = 
        \begin{pmatrix}
            \frac{\sum_i x_{i1} y_i}{g_1} \\
            \vdots \\
            \frac{\sum_i x_{id} y_i}{g_d}
        \end{pmatrix}.
    $$
    
    Therefore, the $i$-th component is
    $$
        \hat{\beta}_{\text{old}, i} = \frac{\sum_j x_{ji} y_j}{g_i}.
    $$

Consider independently sampling columns of the design matrix $X$, retaining each column with probability $p_i$, resulting in a reduced matrix $X'$ with $k$ columns.
The new regression problem is:
$$
    \min_{\beta} \|Y - X'\beta\|_2^2,
$$
with the least squares solution:
$$
    \hat{\beta} = (X'^T X')^{-1} X'^T Y.
$$

Since the columns of $X$ are orthogonal, the columns of $X'$ remain orthogonal after sampling, so $X'^T X'$ is diagonal:
$$
    X'^T X' = \begin{pmatrix}
        g_{j_1} & & \\
        & \ddots & \\
        & & g_{j_k}
    \end{pmatrix},
$$
where $j_1, \dots, j_k$ are the indices of retained columns.

Thus,
$$
    \hat{\beta} = 
    \begin{pmatrix}
        \frac{1}{g_{j_1}} & & \\
        & \ddots & \\
        & & \frac{1}{g_{j_k}}
    \end{pmatrix}
    \begin{pmatrix}
        \sum_i x_{ij_1} y_i \\
        \vdots \\
        \sum_i x_{ij_k} y_i
    \end{pmatrix}
    =
    \begin{pmatrix}
        \frac{\sum_i x_{ij_1} y_i}{g_{j_1}} \\
        \vdots \\
        \frac{\sum_i x_{ij_k} y_i}{g_{j_k}}
    \end{pmatrix}.
$$

We augmented $\hat{\beta}$ such that they have the same dimension, for the dropout column, we fill 0 into $\hat{\beta}_i$, such that their index aligned.
On the other hand, if the $i$-th column is retained, then $\hat{\beta}_i = \hat{\beta}_{\text{old}, i}$.
In this way, the estimator for the $i$-th position satisfies:
$$
    \hat{\beta}_i = 
    \begin{cases}
        \hat{\beta}_{\text{old}, i}, & \text{with probability } p_i, \\
        0, & \text{with probability } 1 - p_i.
    \end{cases}
$$

Thus, its expectation is:
$$
    \mathbb{E}[\hat{\beta}_i] = p_i \cdot \hat{\beta}_{\text{old}, i} + (1 - p_i) \cdot 0 = p_i \hat{\beta}_{\text{old}, i}.
$$

\end{proof}

\begin{theorem}[Expectation of Average Estimator with Multiple Sampling]
    \label{the:expectation}
    Perform $N$ independent random samplings of $\mathcal{Z}$ (each prediction retained with probability $p_i$), obtaining estimators $\hat{\beta}^{(n)}$ each time. 
    Define the average estimate for the $i$-th weight as:
    $$
        \bar{\beta}_i^{(N)} = \frac{1}{N} \sum_{n=1}^N \hat{\beta}_i^{(n)}.
    $$
    Then as $N \to \infty$, we have:
    $$
        \bar{\beta}_i^{(N)} \xrightarrow{\text{a.s.}} p_i \hat{\beta}_{\text{old}, i},
    $$
\end{theorem}

\begin{proof}
    From Lemma~\ref{lem:no_changing_beta}, we know that the expectation of $\hat{\beta}_i^{(n)}$ is:
    $$
        \mathbb{E}[\hat{\beta}_i^{(n)}] = p_i \hat{\beta}_{\text{old}, i}.
    $$

    Since each sampling is independent and identically distributed, by the strong law of large numbers, the sample mean almost surely converges to the expected value:
    $$
        \bar{\beta}_i^{(N)} = \frac{1}{N} \sum_{n=1}^N \hat{\beta}_i^{(n)} \xrightarrow{\text{a.s.}} \mathbb{E}[\hat{\beta}_i^{(n)}] = p_i \hat{\beta}_{\text{old}, i}, \quad \text{as } N \to \infty.
    $$

    Thus, for large $N$, the average weight for the $i$-th model approaches $p_i$ times the original weight.
\end{proof}

\subsubsection{Proof of Theorem~\ref{the:decrease}}

\begin{proof}

    We define $p_i = 1-d_i$, which is the retention probability of each prediction. From Theorem~\ref{the:expectation}, as $N \to \infty$, the absolute average weight of the $i$-th prediction is:
    $$
        |\tilde{\beta}_i| = p_i|\beta_i|
    $$

    Then the proportion of $|\tilde{\beta}_1|$ in the sum of all others is:
    \begin{equation}
        \label{eq:prop}
        r = \frac{p_1 |\beta_1|}{\sum_{i=1}^d p_i |\beta_i|} 
        = \frac{|\beta_1|}{\sum_{i=1}^d \frac{p_i}{p_1} |\beta_i|}
    \end{equation}

    As $\beta_i$ is constant with the change of $\gamma_0$, we focus on:
    $$
        \frac{p_i}{p_1} = \frac{1-d_i}{1-d_1} = \frac{1-\rho_i\gamma_0}{1-\gamma_0}
        = \frac{1-\rho_i}{1-\gamma_0} + \rho_i
    $$

    Note that $\rho_i \leq 1$, then $1-\rho_i \geq 0$, $\frac{p_i}{p_1}$ is monotonically non-decreasing with respect to $\gamma_0$ on the interval [0, 1), and strictly increasing if $\rho_i \neq 1$.

    As a result, each term in the denominator of Equation~\ref{eq:prop} is monotonically non-decreasing with respect to $\gamma_0$. Consequently, $r$ is monotonically non-increasing with respect to $\gamma_0$, and strictly decreasing when there exists some $\rho_i \neq 1$.

    Thus, the theorem is proved.
\end{proof}

%% file: tex/perf_table.tex
\begin{table*}[t]
\fontsize{9}{11}\selectfont
\centering
\begin{tabular}{lccccccc}
\toprule
Datasets & Single-Best & EO & Autostacker & OptDivBO & ES & CMAES & $\textsc{All}$-ES-L1\\
\midrule
fri\_c0\_1000\_5 & 9.5000 & 10.0000 & 10.0000 & 10.0000 & 10.5000 & 9.5000 & 9.0000 \\
fri\_c2\_1000\_5 & 6.5000 & 8.0000 & 7.5000 & 6.0000 & 6.0000 & \textbf{5.5000} & 6.5000 \\
fri\_c4\_1000\_10 & \textbf{4.5000} & 7.0000 & 12.0000 & 8.5000 & 7.5000 & \textbf{4.5000} & 9.5000 \\
fri\_c1\_1000\_50 & 7.0000 & 7.5000 & 4.5000 & 5.0000 & 6.0000 & 7.0000 & 5.5000 \\
credit & 28.5000 & 28.5000 & 29.5000 & 28.0000 & 28.0000 & 28.5000 & \textbf{25.5000} \\
fri\_c3\_1000\_10 & 13.0000 & \underline{9.0000} & 11.5000 & \underline{9.0000} & 10.5000 & 11.0000 & 10.5000 \\
fri\_c0\_1000\_25 & 11.0000 & 11.5000 & 9.0000 & 8.5000 & 10.0000 & 11.0000 & 9.0000 \\
credit-g & 24.5000 & 22.0000 & 22.0000 & 23.0000 & \textbf{19.5000} & 21.0000 & 20.5000 \\
fri\_c3\_1000\_25 & 6.5000 & 6.0000 & 3.0000 & 6.5000 & 6.5000 & 6.5000 & 8.0000 \\
rmftsa\_sleepdata(1) & 50.2439 & 46.8293 & 53.1707 & 46.3415 & 48.2927 & 48.7805 & 46.8293 \\
pc1 & \underline{5.8559} & 6.3063 & 7.2072 & 6.7568 & \underline{5.8559} & \underline{5.8559} & 6.7568 \\
colleges\_aaup & 1.2876 & \textbf{0.8584} & \textbf{0.8584} & \textbf{0.8584} & 1.2876 & 1.2876 & 1.2876 \\
analcatdata\_halloffame & 4.8507 & \textbf{3.7313} & 5.2239 & 4.4776 & 4.8507 & 4.8507 & 4.4776 \\
pc4 & \textbf{7.8767} & 8.2192 & 11.3014 & 11.6438 & 9.9315 & \textbf{7.8767} & 9.5890 \\
cmc & 42.7119 & 48.8136 & 43.3898 & 42.3729 & \underline{41.6949} & 42.7119 & 42.0339 \\
pc3 & 10.2236 & 10.2236 & 10.5431 & 10.5431 & 10.2236 & 10.2236 & 10.2236 \\
semeion & \underline{3.7618} & \underline{3.7618} & 4.7022 & 4.3887 & \underline{3.7618} & 5.0157 & 4.7022 \\
winequality\_red & 29.0625 & 31.5625 & 30.6250 & 31.2500 & 28.7500 & \textbf{27.8125} & 29.3750 \\
mfeat-morphological(1) & 26.0000 & 27.0000 & 21.5000 & 21.0000 & 25.5000 & 26.0000 & 25.5000 \\
mfeat-morphological(2) & \textbf{0.2500} & \textbf{0.2500} & \textbf{0.2500} & \textbf{0.2500} & \textbf{0.2500} & \textbf{0.2500} & \textbf{0.2500} \\
mfeat-karhunen(2) & 0.5000 & 0.2500 & 0.2500 & 0.2500 & \textbf{0.0000} & 0.5000 & \textbf{0.0000} \\
mfeat-karhunen(1) & \underline{2.2500} & \underline{2.2500} & 2.5000 & \underline{2.2500} & 3.0000 & \underline{2.2500} & 3.2500 \\
mfeat-fourier(1) & 17.5000 & 16.5000 & 15.7500 & 14.5000 & 17.0000 & 17.5000 & 17.5000 \\
balloon & \textbf{14.2145} & \textbf{14.2145} & \textbf{14.2145} & \textbf{14.2145} & \textbf{14.2145} & \textbf{14.2145} & \textbf{14.2145} \\
kc1 & 13.5071 & 14.4550 & 12.5592 & 12.5592 & \textbf{12.0853} & 13.5071 & 12.3223 \\
quake & 49.3119 & 46.1009 & 49.5413 & 48.3945 & 48.6239 & 45.6422 & \underline{45.4128} \\
space\_ga & 13.8264 & 17.2026 & 12.5402 & 12.5402 & 13.3441 & 14.1479 & 12.0579 \\
splice & 3.9185 & 3.9185 & 3.6050 & 4.2320 & 3.4483 & 3.9185 & 3.7618 \\
kr-vs-kp & 0.6250 & 0.6250 & 0.7812 & 0.7812 & \textbf{0.4687} & 0.6250 & 0.6250 \\
hypothyroid(2) & \textbf{0.2649} & \textbf{0.2649} & \textbf{0.2649} & \textbf{0.2649} & \textbf{0.2649} & \textbf{0.2649} & \textbf{0.2649} \\
analcatdata\_supreme & 0.6165 & \textbf{0.4932} & 0.7398 & 0.6165 & \textbf{0.4932} & 0.6165 & 0.7398 \\
waveform-5000(2) & 10.9000 & 9.9000 & 18.3000 & 9.5000 & 10.1000 & 9.6000 & \textbf{9.1000} \\
delta\_ailerons & 6.0309 & \textbf{5.7504} & 6.0309 & \underline{5.8906} & 6.2412 & 6.0309 & 5.9607 \\
isolet & 3.0769 & 3.1410 & 2.9487 & 2.8205 & 3.0769 & 3.0128 & 3.0769 \\
puma32H & 8.6028 & 12.3246 & 8.0537 & 8.2977 & 9.3350 & 8.7248 & 8.7858 \\
cpu\_small & 7.3826 & \textbf{7.0775} & 7.7486 & 7.9317 & 7.5656 & 7.5046 & 7.5046 \\
puma8NH & 17.2056 & 17.9988 & 16.7175 & 17.1446 & 16.7175 & 17.5717 & 16.9616 \\
cpu\_act & 6.8944 & 6.4674 & 6.5284 & 6.4674 & 6.3453 & 6.8944 & 6.5284 \\
jm1 & 18.1442 & 18.0524 & 17.9146 & 18.1442 & 17.2715 & 17.7308 & 17.8227 \\
sylva\_prior & 0.5210 & 0.4515 & 0.5210 & 0.4168 & 0.5210 & 0.5210 & 0.5210 \\
letter(2) & 0.1000 & 0.0750 & 0.0500 & \underline{0.0250} & \underline{0.0250} & 0.0500 & \underline{0.0250} \\
house\_16H & 10.4016 & 10.0505 & 10.0066 & 9.6774 & 9.7652 & 10.0944 & 9.8310 \\
kropt & \underline{7.2701} & 12.0991 & 9.1411 & 8.5709 & 7.8582 & 7.2880 & 14.9679 \\
adult-census & 13.2197 & 15.7070 & 13.0355 & 13.1890 & 13.3272 & \underline{12.8512} & 12.9894 \\
amazon\_employee & 4.8062 & 5.2487 & 4.6689 & 4.6842 & 4.6689 & 4.8062 & 4.8062 \\
mv & 0.0491 & 0.1104 & 0.0858 & 0.0613 & 0.1104 & 0.0491 & 0.0368 \\
adult & 13.9114 & 14.4846 & 13.9114 & 14.0035 & 13.8704 & \textbf{13.7271} & 13.9728 \\
mnist\_784 & 1.8286 & 1.8286 & 2.0500 & 1.3714 & 1.6214 & 1.5000 & 1.5357 \\
vehicle\_sensIT & 12.9605 & 12.9605 & 12.5444 & \underline{12.3566} & 12.9402 & 12.9554 & 12.9301 \\
covertype & 11.5766 & 11.1871 & 9.1852 & \underline{9.1308} & 10.1363 & 9.9733 & 11.5766 \\
\bottomrule
\end{tabular}
\caption{Test Error(\%) - CLS(Part I): The test score for each method. The best methods per dataset are shown in bold, while the second-best methods are underlined.}
\label{tab:cls_results_details_part1}
\end{table*}
\begin{table*}[t]
\fontsize{9}{11}\selectfont
\centering
\begin{tabular}{lcccccc}
\toprule
Datasets & $\textsc{All}$-ES-L2 & $\textsc{All}$-Linear-L1 & $\textsc{All}$-Linear-L2 & $\textsc{Best}$-ES-L1 & $\textsc{Best}$-ES-L2 & $\textsc{Best}$-Linear-L1\\
\midrule
fri\_c0\_1000\_5 & 9.0000 & \textbf{8.5000} & 9.5000 & 12.0000 & 10.5000 & 12.0000 \\
fri\_c2\_1000\_5 & 6.0000 & 6.0000 & 6.5000 & 6.5000 & 7.0000 & 7.5000 \\
fri\_c4\_1000\_10 & 9.0000 & 5.5000 & 9.5000 & 8.5000 & 9.5000 & 9.5000 \\
fri\_c1\_1000\_50 & \underline{3.5000} & \textbf{3.0000} & 4.5000 & 4.5000 & 4.0000 & 5.0000 \\
credit & 27.5000 & \underline{26.5000} & 31.5000 & 27.0000 & 29.5000 & 28.5000 \\
fri\_c3\_1000\_10 & \textbf{7.5000} & \underline{9.0000} & 11.5000 & 12.0000 & 10.0000 & 10.5000 \\
fri\_c0\_1000\_25 & 9.0000 & \textbf{8.0000} & \textbf{8.0000} & 9.0000 & 9.0000 & 8.5000 \\
credit-g & 20.0000 & 21.0000 & 21.0000 & 20.5000 & 24.0000 & 20.0000 \\
fri\_c3\_1000\_25 & \underline{1.5000} & 4.5000 & 4.5000 & 6.5000 & 5.5000 & 5.5000 \\
rmftsa\_sleepdata(1) & 47.8049 & \textbf{45.3659} & 49.2683 & 46.8293 & 47.8049 & 46.8293 \\
pc1 & 8.5586 & 7.6577 & 7.2072 & 6.3063 & \underline{5.8559} & 6.7568 \\
colleges\_aaup & \textbf{0.8584} & 1.2876 & 1.2876 & 1.2876 & 1.2876 & 1.7167 \\
analcatdata\_halloffame & 4.8507 & \underline{4.1045} & 4.4776 & 4.8507 & 4.4776 & 4.8507 \\
pc4 & 9.5890 & 10.6164 & 9.5890 & 9.2466 & 9.5890 & 9.5890 \\
cmc & 43.7288 & 43.3898 & 46.7797 & 42.0339 & 44.0678 & \textbf{40.0000} \\
pc3 & 10.5431 & 10.5431 & 9.9042 & 10.2236 & \underline{9.5847} & 10.5431 \\
semeion & 4.3887 & 4.7022 & 4.3887 & 4.0752 & 4.3887 & \textbf{3.4483} \\
winequality\_red & 30.0000 & 32.8125 & 33.4375 & \underline{28.1250} & 31.2500 & 29.0625 \\
mfeat-morphological(1) & \underline{19.2500} & 20.5000 & 22.7500 & 25.7500 & 21.2500 & 21.2500 \\
mfeat-morphological(2) & \textbf{0.2500} & \textbf{0.2500} & \textbf{0.2500} & \textbf{0.2500} & \textbf{0.2500} & \textbf{0.2500} \\
mfeat-karhunen(2) & 0.2500 & 0.2500 & 0.2500 & 0.2500 & \textbf{0.0000} & 0.2500 \\
mfeat-karhunen(1) & 2.5000 & \underline{2.2500} & 2.7500 & \underline{2.2500} & \underline{2.2500} & 3.0000 \\
mfeat-fourier(1) & 17.5000 & \underline{12.2500} & 13.5000 & 17.5000 & 17.5000 & 12.7500 \\
balloon & 15.2120 & \textbf{14.2145} & \textbf{14.2145} & \textbf{14.2145} & \textbf{14.2145} & \textbf{14.2145} \\
kc1 & \textbf{12.0853} & \textbf{12.0853} & 12.5592 & \textbf{12.0853} & 12.3223 & 12.3223 \\
quake & 47.7064 & 49.3119 & \underline{45.4128} & \textbf{44.7248} & 49.3119 & \underline{45.4128} \\
space\_ga & 12.8617 & 12.2186 & \textbf{11.5756} & 12.7010 & 12.7010 & 13.0225 \\
splice & \underline{3.2915} & \textbf{2.8213} & \underline{3.2915} & 3.9185 & 3.7618 & 3.4483 \\
kr-vs-kp & 0.6250 & 0.6250 & 0.6250 & 0.6250 & \textbf{0.4687} & 0.6250 \\
hypothyroid(2) & \textbf{0.2649} & \textbf{0.2649} & \textbf{0.2649} & \textbf{0.2649} & \textbf{0.2649} & 0.3974 \\
analcatdata\_supreme & 0.7398 & 0.7398 & 0.7398 & 0.7398 & 0.7398 & 0.7398 \\
waveform-5000(2) & \textbf{9.1000} & 10.1000 & 9.5000 & 9.9000 & 9.9000 & 10.1000 \\
delta\_ailerons & 5.9607 & 6.1010 & 6.3815 & \underline{5.8906} & 5.9607 & 5.9607 \\
isolet & 2.9487 & 2.9487 & \underline{2.6923} & 3.0128 & \textbf{2.6282} & 3.2692 \\
puma32H & 7.8707 & \textbf{7.1385} & 7.8707 & 8.4198 & 7.8096 & 8.9689 \\
cpu\_small & 7.4436 & 7.5046 & 7.5046 & 7.3215 & 7.6876 & 7.5656 \\
puma8NH & \underline{16.6565} & \textbf{16.4735} & 17.2056 & 16.9005 & 16.8395 & 16.8395 \\
cpu\_act & 6.2843 & \textbf{5.7962} & 6.1013 & 6.5894 & 6.1623 & 6.7724 \\
jm1 & 17.5011 & 17.0877 & 18.0983 & \textbf{16.8581} & 17.8227 & \underline{17.0418} \\
sylva\_prior & \textbf{0.3821} & \textbf{0.3821} & 0.4515 & 0.5557 & 0.4863 & 0.4515 \\
letter(2) & \underline{0.0250} & 0.0500 & \underline{0.0250} & 0.0500 & 0.0500 & \textbf{0.0000} \\
house\_16H & 9.9408 & \textbf{9.5458} & 9.9627 & 9.9408 & 9.9846 & 9.7871 \\
kropt & 8.3393 & 11.0478 & 8.7847 & 14.8076 & 10.4419 & 13.5424 \\
adult-census & 12.8819 & 13.0508 & 13.0815 & 13.1122 & 13.0355 & 13.2044 \\
amazon\_employee & 4.6384 & \textbf{4.5468} & 4.7910 & 4.8520 & 4.7147 & 4.7605 \\
mv & 0.0368 & 0.1226 & \underline{0.0245} & 0.0858 & 0.0368 & 0.0736 \\
adult & 13.9625 & 13.8806 & 14.6279 & 13.9318 & 13.8806 & 14.0035 \\
mnist\_784 & 1.4071 & \underline{1.3643} & 1.4000 & 1.5857 & 1.5643 & 1.5929 \\
vehicle\_sensIT & 12.4480 & 12.8438 & 12.4023 & 12.8540 & 12.5647 & 13.1534 \\
covertype & 11.5766 & 12.7678 & 11.4362 & 11.5766 & 11.5766 & 9.6562 \\
\bottomrule
\end{tabular}
\caption{Test Error(\%) - CLS(Part II): The test score for each method. The best methods per dataset are shown in bold, while the second-best methods are underlined.}
\label{tab:cls_results_details_part2}
\end{table*}
\begin{table*}[t]
\fontsize{9}{11}\selectfont
\centering
\begin{tabular}{lccc}
\toprule
Datasets & $\textsc{Best}$-Linear-L2 & AutoGluon- & \sys\\
\midrule
fri\_c0\_1000\_5 & 11.0000 & 9.0000 & \textbf{8.5000} \\
fri\_c2\_1000\_5 & 6.5000 & 6.3595 & \textbf{5.5000} \\
fri\_c4\_1000\_10 & 8.0000 & 9.3736 & 7.5000 \\
fri\_c1\_1000\_50 & \underline{3.5000} & 5.4807 & \underline{3.5000} \\
credit & 29.0000 & 26.9174 & 27.5000 \\
fri\_c3\_1000\_10 & 11.0000 & 9.6385 & \underline{9.0000} \\
fri\_c0\_1000\_25 & 9.0000 & 9.0000 & 8.5000 \\
credit-g & 22.5000 & 19.9686 & \textbf{19.5000} \\
fri\_c3\_1000\_25 & 3.5000 & \textbf{0.0000} & 4.0000 \\
rmftsa\_sleepdata(1) & 48.2927 & 47.7064 & \underline{45.8537} \\
pc1 & \textbf{5.4054} & 8.3585 & \underline{5.8559} \\
colleges\_aaup & 1.2876 & 1.2184 & \textbf{0.8584} \\
analcatdata\_halloffame & 4.8507 & 4.4593 & 4.4776 \\
pc4 & 9.5890 & 9.5890 & 9.5890 \\
cmc & 42.3729 & 41.9174 & 42.0339 \\
pc3 & \textbf{8.9457} & 10.4732 & 9.9042 \\
semeion & 4.0752 & 4.6679 & \underline{3.7618} \\
winequality\_red & 31.5625 & 29.3248 & 29.0625 \\
mfeat-morphological(1) & 20.2500 & \textbf{19.1617} & 20.5000 \\
mfeat-morphological(2) & \textbf{0.2500} & \textbf{0.2500} & \textbf{0.2500} \\
mfeat-karhunen(2) & \textbf{0.0000} & \textbf{0.0000} & 0.2500 \\
mfeat-karhunen(1) & 2.5000 & 3.0563 & \textbf{2.0000} \\
mfeat-fourier(1) & \textbf{11.7500} & 17.5000 & \underline{12.2500} \\
balloon & \textbf{14.2145} & 15.1321 & \textbf{14.2145} \\
kc1 & 12.3223 & 12.3215 & \textbf{12.0853} \\
quake & 47.2477 & 47.6466 & 46.5596 \\
space\_ga & 12.3794 & 12.7126 & \underline{11.8971} \\
splice & 3.6050 & 3.6826 & \underline{3.2915} \\
kr-vs-kp & 0.6250 & 0.6250 & 0.6250 \\
hypothyroid(2) & \textbf{0.2649} & \textbf{0.2649} & \textbf{0.2649} \\
analcatdata\_supreme & 0.7398 & 0.7398 & 0.6165 \\
waveform-5000(2) & 9.7000 & \textbf{9.1000} & 9.4000 \\
delta\_ailerons & 5.9607 & 5.9607 & \underline{5.8906} \\
isolet & 2.9487 & 2.9125 & \underline{2.6923} \\
puma32H & 8.5418 & 7.8487 & \underline{7.5046} \\
cpu\_small & 7.6876 & 7.4394 & \underline{7.1995} \\
puma8NH & \underline{16.6565} & 16.9588 & \underline{16.6565} \\
cpu\_act & 6.5284 & 6.4814 & \textbf{5.7962} \\
jm1 & 17.7768 & 17.4707 & 17.5011 \\
sylva\_prior & 0.4515 & 0.4925 & \textbf{0.3821} \\
letter(2) & 0.0500 & \underline{0.0250} & \underline{0.0250} \\
house\_16H & 9.8969 & 9.8224 & \underline{9.6555} \\
kropt & 10.9230 & \textbf{7.1768} & 8.6778 \\
adult-census & 13.1276 & 12.8646 & \textbf{12.7284} \\
amazon\_employee & 4.7452 & 4.6271 & \underline{4.5926} \\
mv & 0.0368 & 0.0368 & \textbf{0.0123} \\
adult & 13.9011 & 13.9622 & \underline{13.8499} \\
mnist\_784 & 1.4429 & 1.3976 & \textbf{1.3286} \\
vehicle\_sensIT & 12.7220 & 12.4078 & \textbf{12.3211} \\
covertype & 9.5792 & 11.5766 & \textbf{9.0991} \\
\bottomrule
\end{tabular}
\caption{Test Error(\%) - CLS(Part III): The test score for each method. The best methods per dataset are shown in bold, while the second-best methods are underlined.}
\label{tab:cls_results_details_part3}
\end{table*}
\begin{table*}[t]
\fontsize{9}{11}\selectfont
\centering
\begin{tabular}{lccccccc}
\toprule
Datasets & Single-Best & EO & Autostacker & OptDivBO & ES & CMAES & $\textsc{All}$-ES-L1\\
\midrule
strikes & 3.4118e+05 & 3.4072e+05 & 3.3048e+05 & 3.3038e+05 & 3.3111e+05 & \textbf{3.0473e+05} & 3.2037e+05 \\
disclosure\_x\_noise & 8.4511e+08 & 8.4815e+08 & 8.5393e+08 & 8.5419e+08 & 8.7719e+08 & 8.8265e+08 & 8.2596e+08 \\
disclosure\_z & 6.2967e+08 & 5.9922e+08 & 6.7199e+08 & 6.2887e+08 & 7.6631e+08 & 7.6327e+08 & 6.2287e+08 \\
stock & 2.7339e-01 & 3.0387e-01 & 3.2874e-01 & 3.0445e-01 & 2.9844e-01 & 3.0871e-01 & 3.1732e-01 \\
socmob & \textbf{1.0809e+02} & \underline{1.0937e+02} & 2.6734e+02 & 1.7550e+02 & 1.1101e+02 & 1.1292e+02 & 1.7066e+02 \\
Moneyball & 4.5700e+02 & 8.2509e+02 & 4.6171e+02 & 4.5645e+02 & \textbf{4.3366e+02} & \underline{4.4980e+02} & 4.5700e+02 \\
insurance & 1.8019e+07 & 1.8779e+07 & 1.7728e+07 & \underline{1.7609e+07} & 1.8966e+07 & 1.9077e+07 & 1.8069e+07 \\
weather\_izmir & 1.4356e+00 & 1.6041e+00 & 1.4121e+00 & 1.3884e+00 & 1.3880e+00 & 1.3947e+00 & \textbf{1.3697e+00} \\
us\_crime & 1.8158e-02 & 1.7748e-02 & 1.8093e-02 & 1.8443e-02 & \textbf{1.7185e-02} & \underline{1.7254e-02} & 1.7655e-02 \\
debutanizer & 4.3706e-03 & 3.6395e-03 & 3.5604e-03 & 3.7100e-03 & 3.8771e-03 & 3.8058e-03 & 3.9273e-03 \\
space\_ga & 1.0894e-02 & 9.5217e-03 & 8.7566e-03 & 8.7410e-03 & 9.0999e-03 & 9.5045e-03 & 8.7747e-03 \\
mtp & 1.2382e-02 & 1.2274e-02 & 1.1829e-02 & 1.2209e-02 & 1.1864e-02 & 1.2054e-02 & 1.1889e-02 \\
wind & 8.7384e+00 & 8.7125e+00 & 8.2138e+00 & 8.2040e+00 & 8.2374e+00 & 8.3029e+00 & 8.1847e+00 \\
bank32nh & 6.2034e-03 & 6.1807e-03 & 5.8875e-03 & 6.0119e-03 & 5.7367e-03 & 5.6970e-03 & 5.7160e-03 \\
bank8FM & 9.3215e-04 & 1.0966e-03 & 8.8017e-04 & 8.8556e-04 & 9.0094e-04 & 9.0198e-04 & 8.8128e-04 \\
cpu\_act & 6.9071e+00 & 6.1408e+00 & 5.5837e+00 & 4.8873e+00 & 5.4768e+00 & 5.2258e+00 & 5.4218e+00 \\
cpu\_small & 8.3812e+00 & 8.6239e+00 & 7.9154e+00 & 7.9213e+00 & 7.6868e+00 & 7.8099e+00 & 7.7066e+00 \\
kin8nm & 7.1055e-03 & 9.8226e-03 & 5.0617e-03 & \underline{4.9591e-03} & 6.2153e-03 & 6.8652e-03 & 6.3535e-03 \\
puma32H & 5.0607e-05 & 6.1511e-05 & 4.8155e-05 & \textbf{4.7037e-05} & 4.9707e-05 & 4.9512e-05 & 4.8388e-05 \\
puma8NH & 1.0025e+01 & 1.0540e+01 & 9.9420e+00 & 9.8439e+00 & 9.8239e+00 & 9.8729e+00 & 9.8330e+00 \\
sulfur & 7.4291e-04 & \underline{6.2906e-04} & 8.6869e-04 & 9.2979e-04 & \textbf{6.0697e-04} & 6.5443e-04 & 6.7434e-04 \\
rainfall\_bangladesh & 1.5629e+04 & 1.5449e+04 & 1.4735e+04 & \underline{1.4692e+04} & 1.5351e+04 & 1.5342e+04 & 1.5270e+04 \\
kc\_house\_data & 1.9822e+10 & 1.8827e+10 & 2.0677e+10 & \underline{1.8813e+10} & 1.9452e+10 & 1.9954e+10 & 1.9868e+10 \\
house\_8L & 7.7687e+08 & 7.6990e+08 & 7.4949e+08 & 7.7509e+08 & 7.4441e+08 & 7.5221e+08 & 7.4307e+08 \\
NewFuelCar & 6.4424e-02 & 6.5676e-02 & 5.7260e-02 & 5.6136e-02 & 5.7617e-02 & 5.7621e-02 & 5.7079e-02 \\
electricity\_prices\_ICON & 4.3738e+02 & 4.2373e+02 & 4.0376e+02 & 3.8563e+02 & 4.0931e+02 & 4.0980e+02 & 4.1112e+02 \\
OnlineNewsPopularity & 6.7930e+07 & 6.7818e+07 & 6.7858e+07 & 6.7858e+07 & 6.7661e+07 & 6.7636e+07 & 6.7509e+07 \\
2dplanes & 9.9194e-01 & 4.3553e+00 & 9.9361e-01 & 9.9224e-01 & \textbf{9.9162e-01} & \underline{9.9165e-01} & 9.9206e-01 \\
mv & 6.5792e-03 & 6.5792e-03 & 1.7953e-03 & \underline{1.7268e-03} & 2.9186e-03 & 3.9138e-03 & 2.2636e-03 \\
black\_friday & 1.1918e+07 & 1.1918e+07 & 1.1895e+07 & 1.1817e+07 & 1.1833e+07 & 1.1831e+07 & \underline{1.1817e+07} \\
\bottomrule
\end{tabular}
\caption{Test Error - REG(Part I): The test score for each method. The best methods per dataset are shown in bold, while the second-best methods are underlined.}
\label{tab:rgs_results_details_part1}
\end{table*}
\begin{table*}[t]
\fontsize{9}{11}\selectfont
\centering
\begin{tabular}{lcccccc}
\toprule
Datasets & $\textsc{All}$-ES-L2 & $\textsc{All}$-Linear-L1 & $\textsc{All}$-Linear-L2 & $\textsc{Best}$-ES-L1 & $\textsc{Best}$-ES-L2 & $\textsc{Best}$-Linear-L1\\
\midrule
strikes & 3.1170e+05 & \underline{3.0639e+05} & 1.0279e+06 & 3.2392e+05 & 3.2338e+05 & 3.2621e+05 \\
disclosure\_x\_noise & 8.6024e+08 & \textbf{8.1313e+08} & 8.6291e+08 & \underline{8.2466e+08} & 8.5209e+08 & 8.2637e+08 \\
disclosure\_z & 5.9914e+08 & 6.1621e+08 & \underline{5.9904e+08} & 6.1290e+08 & 6.1917e+08 & 6.1421e+08 \\
stock & 3.2968e-01 & 3.1504e-01 & 3.3426e-01 & 2.6611e-01 & \textbf{2.5553e-01} & 2.6503e-01 \\
socmob & 1.2850e+02 & 1.6114e+02 & 1.2702e+02 & 2.0405e+02 & 2.0284e+02 & 1.8475e+02 \\
Moneyball & 4.5700e+02 & 4.6038e+02 & 4.5565e+02 & 4.5700e+02 & 4.5700e+02 & 4.5401e+02 \\
insurance & 1.7691e+07 & 1.8086e+07 & 1.7632e+07 & 1.7875e+07 & 1.7726e+07 & 1.7860e+07 \\
weather\_izmir & 1.3768e+00 & \underline{1.3715e+00} & 1.3770e+00 & 1.4216e+00 & 1.4115e+00 & 1.3996e+00 \\
us\_crime & 1.7625e-02 & 1.7658e-02 & 1.7621e-02 & 1.7744e-02 & 1.7913e-02 & 1.7792e-02 \\
debutanizer & 3.5519e-03 & 3.8544e-03 & 3.5098e-03 & 4.9969e-03 & 3.5002e-03 & 4.7827e-03 \\
space\_ga & \underline{8.1149e-03} & 8.4248e-03 & 8.1879e-03 & 9.0069e-03 & 8.3849e-03 & 8.9144e-03 \\
mtp & 1.1819e-02 & 1.1858e-02 & 1.1805e-02 & 1.1969e-02 & 1.1755e-02 & 1.1939e-02 \\
wind & \underline{8.0532e+00} & 8.1497e+00 & 8.0664e+00 & 8.5639e+00 & 8.5666e+00 & 8.5599e+00 \\
bank32nh & 5.7391e-03 & 5.7144e-03 & 5.7371e-03 & \textbf{5.6783e-03} & 5.7042e-03 & \underline{5.6925e-03} \\
bank8FM & 8.6807e-04 & 8.8070e-04 & 8.6854e-04 & 9.0888e-04 & 8.9631e-04 & 9.1144e-04 \\
cpu\_act & 4.2041e+00 & 5.4098e+00 & 4.2214e+00 & 5.3814e+00 & \underline{4.2026e+00} & 5.1073e+00 \\
cpu\_small & \textbf{6.1522e+00} & 7.7043e+00 & \underline{6.1565e+00} & 7.0833e+00 & 6.2068e+00 & 7.0343e+00 \\
kin8nm & 5.0694e-03 & 6.0920e-03 & 5.0629e-03 & 6.9057e-03 & 5.3107e-03 & 6.4907e-03 \\
puma32H & 4.7965e-05 & 4.8519e-05 & 4.8045e-05 & 4.8455e-05 & 4.7824e-05 & 4.8868e-05 \\
puma8NH & 9.7711e+00 & \textbf{9.7532e+00} & \underline{9.7689e+00} & 9.8977e+00 & 9.8063e+00 & 9.8508e+00 \\
sulfur & 7.5556e-04 & 6.8126e-04 & 7.5840e-04 & 7.0458e-04 & 7.2464e-04 & 6.9111e-04 \\
rainfall\_bangladesh & 1.5373e+04 & 1.5281e+04 & 1.5408e+04 & 1.5417e+04 & 1.4941e+04 & 1.5397e+04 \\
kc\_house\_data & 2.1081e+10 & 1.9768e+10 & 2.1064e+10 & 1.9560e+10 & 1.9502e+10 & 1.9900e+10 \\
house\_8L & 7.5804e+08 & 7.4421e+08 & 7.5797e+08 & 7.5197e+08 & 7.5404e+08 & 7.4965e+08 \\
NewFuelCar & \underline{5.4681e-02} & 5.6992e-02 & \textbf{5.4533e-02} & 5.8940e-02 & 5.8370e-02 & 5.8970e-02 \\
electricity\_prices\_ICON & 3.8141e+02 & 4.1110e+02 & \underline{3.8131e+02} & 4.1964e+02 & 3.9244e+02 & 4.1944e+02 \\
OnlineNewsPopularity & \textbf{6.7155e+07} & 6.7675e+07 & 6.7540e+07 & 6.7707e+07 & 6.7869e+07 & 6.7728e+07 \\
2dplanes & 1.0750e+00 & 9.9207e-01 & 1.0752e+00 & 9.9183e-01 & 9.9585e-01 & 9.9172e-01 \\
mv & 1.8193e-03 & 2.2777e-03 & 1.8124e-03 & 3.0865e-03 & 2.1893e-03 & 2.9955e-03 \\
black\_friday & 1.1981e+07 & 1.1817e+07 & 1.1978e+07 & 1.1895e+07 & 1.1896e+07 & 1.1898e+07 \\
\bottomrule
\end{tabular}
\caption{Test Error - REG(Part II): The test score for each method. The best methods per dataset are shown in bold, while the second-best methods are underlined.}
\label{tab:rgs_results_details_part2}
\end{table*}
\begin{table*}[t]
\fontsize{9}{11}\selectfont
\centering
\begin{tabular}{lccc}
\toprule
Datasets & $\textsc{Best}$-Linear-L2 & AutoGluon- & \sys\\
\midrule
strikes & 3.2460e+05 & 3.1159e+05 & 3.1458e+05 \\
disclosure\_x\_noise & 8.5190e+08 & 8.5873e+08 & 8.4298e+08 \\
disclosure\_z & 6.1692e+08 & \textbf{5.9697e+08} & 6.1043e+08 \\
stock & \underline{2.5841e-01} & 3.2961e-01 & 2.9864e-01 \\
socmob & 2.0080e+02 & 1.2468e+02 & 1.3939e+02 \\
Moneyball & 4.5371e+02 & 4.5700e+02 & 4.5461e+02 \\
insurance & 1.7729e+07 & 1.7691e+07 & \textbf{1.7543e+07} \\
weather\_izmir & 1.4111e+00 & 1.3766e+00 & 1.3758e+00 \\
us\_crime & 1.8002e-02 & 1.7653e-02 & 1.7698e-02 \\
debutanizer & 3.4669e-03 & \underline{3.1567e-03} & \textbf{3.0108e-03} \\
space\_ga & 8.3903e-03 & \textbf{8.1136e-03} & 8.7202e-03 \\
mtp & \underline{1.1738e-02} & 1.1875e-02 & \textbf{1.1614e-02} \\
wind & 8.5606e+00 & \textbf{8.0476e+00} & 8.1334e+00 \\
bank32nh & 5.6977e-03 & 5.7386e-03 & 5.7100e-03 \\
bank8FM & 8.9523e-04 & \underline{8.6743e-04} & \textbf{8.6717e-04} \\
cpu\_act & \textbf{4.2022e+00} & 5.4057e+00 & 4.8873e+00 \\
cpu\_small & 6.2091e+00 & 7.5701e+00 & 7.4696e+00 \\
kin8nm & 5.2556e-03 & 5.0111e-03 & \textbf{4.8166e-03} \\
puma32H & 4.7801e-05 & 4.8314e-05 & \underline{4.7045e-05} \\
puma8NH & 9.8042e+00 & 9.8319e+00 & 9.7779e+00 \\
sulfur & 7.2344e-04 & 6.6906e-04 & 6.6811e-04 \\
rainfall\_bangladesh & 1.4921e+04 & 1.5370e+04 & \textbf{1.4598e+04} \\
kc\_house\_data & 1.9581e+10 & 1.9839e+10 & \textbf{1.8714e+10} \\
house\_8L & 7.5418e+08 & \underline{7.4225e+08} & \textbf{7.3476e+08} \\
NewFuelCar & 5.7688e-02 & 5.7062e-02 & 5.5005e-02 \\
electricity\_prices\_ICON & 3.9112e+02 & \textbf{3.8040e+02} & 3.8629e+02 \\
OnlineNewsPopularity & 6.7851e+07 & 6.7501e+07 & \underline{6.7486e+07} \\
2dplanes & 9.9713e-01 & 1.0671e+00 & 9.9183e-01 \\
mv & 2.2100e-03 & 1.7811e-03 & \textbf{1.7101e-03} \\
black\_friday & 1.1900e+07 & 1.1967e+07 & \textbf{1.1795e+07} \\
\bottomrule
\end{tabular}
\caption{Test Error -REG(Part III): The test score for each method. The best methods per dataset are shown in bold, while the second-best methods are underlined.}
\label{tab:rgs_results_details_part3}
\end{table*}